\documentclass[twoside]{article}

\usepackage[accepted]{aistats2023}

\usepackage[round]{natbib}

\usepackage{xcolor}
\usepackage{amsfonts}
\usepackage{amsthm}
\usepackage{bbm}
\usepackage{float}
\usepackage{dsfont}

\usepackage{graphicx}

\usepackage{hyperref}
\hypersetup{
    colorlinks,
    citecolor=black,
    linkcolor=black,
    filecolor=black,
    urlcolor=black
}

\usepackage{enumitem}

\newtheorem{prop}{Proposition}[section]
\newtheorem{lemma}{Lemma}[section]
\newtheorem{theorem}{Theorem}[section]
\newtheorem{remark}{Remark}[section]

\newcommand{\X}{\mathbf{X}}
\newcommand\indep{\protect\mathpalette{\protect\independenT}{\perp}}
\def\independenT#1#2{\mathrel{\rlap{$#1#2$}\mkern3mu{#1#2}}}
\newcommand\norm[1]{\left\lVert#1\right\rVert}

\begin{document}

\twocolumn[

\aistatstitle{Rank-Based Causal Discovery for Post-Nonlinear Models}

\aistatsauthor{ Grigor Keropyan \And David Strieder \And  Mathias Drton }

\aistatsaddress{
Technical University of Munich \And 
Technical University of Munich \\ Munich Center for Machine Learning  \And 
Technical University of Munich \\ \ Munich Center for Machine Learning  } ]

\begin{abstract}
    Learning causal relationships from empirical observations is a central task in scientific research. A common method is to employ structural causal models that postulate noisy functional relations among a set of interacting variables. To ensure unique identifiability of causal directions, researchers consider restricted subclasses of structural causal models.  Post-nonlinear (PNL) causal models constitute one of the most flexible options for such restricted subclasses, containing in particular the popular additive noise models as a further subclass.
    However, learning PNL models is not well studied beyond the bivariate case. The existing methods learn non-linear functional relations by minimizing residual dependencies and subsequently test independence from residuals to determine causal orientations.  However, these methods can be prone to overfitting and, thus, difficult to tune appropriately in practice.  As an alternative, we propose a new approach for PNL causal discovery that uses rank-based methods to estimate the functional parameters.  This new approach exploits natural invariances of PNL models and disentangles the estimation of the non-linear functions from the independence tests used to find causal orientations.  We prove consistency of our method and validate our results in numerical experiments.
\end{abstract}

\section{INTRODUCTION}
\label{section:introduction}

Discovering the causal structure of complex systems is an important question in various disciplines such as biology, economics, clinical medicine, or neuroscience  \citep{biology_ref, review_ref, oxford_ref}. The gold standard approach to exploration of causal relations is to perform controlled experiments in which researchers externally intervene in the system and observe the resulting changes to variables of interest. However, in many applications controlled experiments are not feasible due to high cost or for ethical reasons.  In such cases, causal discovery based on only observational data can be a useful tool \citep{randomized_experiments_ref}. 

A common tool for modeling causal relations are Structural Equation Models (SEMs). In their general form, SEMs postulate noisy functional relationships between a set of interacting variables. In the fully general setting, causal discovery methods such as constraint-based and score-based methods can identify the underlying causal structure only up to Markov equivalence classes \citep{Spirtes2000}. Thus, the literature has also considered many restricted subclasses that enable unique identification \citep{global_idf_linear_gaussian, hoyer_2009, ANM_2011, idf_bivariate}.  In this realm, post-nonlinear (PNL) causal models constitute one of the most general approaches.  They are  identifiable from the joint distribution under mild assumptions  \citep{idf_bivariate, ANM_2011} and yet offer a rather flexible framework for  modeling complex non-linear causal systems. 

Existing methods for bivariate PNL causal discovery are based on estimating the functional relations by minimizing independence criteria (HSIC, mutual information, etc.) between the noise and potential parents in a first step, and performing independence tests to determine the causal structure in a second step \citep{idf_bivariate, pnl_bv_2020}. However, minimizing dependence to subsequently test for independence leads to potential  overfitting and thus limits the PNL approach. Another approach considered by \citet{pnl_optimal_transport_2022} employs Optimal Transport theory for bivariate post-nonlinear and additive noise causal discovery, but it is not evident how to generalize their method to multivariate models. 
To our knowledge the only work that deals with multivariate PNL models is \cite{pnl_mult_2022}, where the authors generalize the bivariate method from \cite{pnl_bv_2020} based on minimizing dependence and subsequently testing for independence. 

In this article we present a new method for multivariate causal discovery in PNL models that disentangles the two tasks of learning the functional relations and learning the causal structure.  Our method continues to learn the latter with the help of independence tests, but it employs  rank-based methods to learn the functional relations and, in this way, avoids overfitting issues.

The remainder of the paper is organized as follows. In Section \ref{section:pnl_regression} we introduce the PNL rank regression methods that we use to learn the functional relations. We study the special case of linearity in the inner function and show consistency of our proposed rank-based functional parameter estimates. This special case includes general nonlinear functions using basis expansions. In Section \ref{section:learn_pnl} we discuss the causal order learning routine using our proposed rank-based estimates in a recursive process that finds sink nodes by independence testing. Furthermore, we show consistency of the causal order estimation for PNL models and present the results of a simulation study in Section \ref{section:simulations}, where we compare our method to existing causal learning methods. 
Section \ref{section:conclusion} concludes the paper.

\section{PNL RANK REGRESSION}
\label{section:pnl_regression}

In this section we introduce the rank-based estimators that we use in the first step of our proposed causal learning algorithm. The goal is to employ these rank-based estimators to infer the functional relations among the variables and to obtain estimates of the stochastic noise terms in the model. In the second step, we use the estimated noise terms to test for independence. 

Suppose we observe a sample of $n$ independent copies $(X_1,Y_1), \dots , (X_n, Y_n)$ of a random vector $(X,Y)$. We assume the data generating process follows a PNL model, that is, the response variable  $Y \in \mathbb{R}$ is given by
\begin{equation}
    \label{model:pnl_regression}
    h(Y) = g(X) + \varepsilon,
\end{equation} 
where $X \in \mathbb{R}^m$ is a continuous random vector and the stochastic error term $\varepsilon$ has mean zero with unknown continuous distribution, independent of $X$. Furthermore, we assume that the function $h : \mathbb{R} \to \mathbb{R}$ is continuous and strictly increasing (thus, invertible) whereas $g : \mathbb{R}^m \to \mathbb{R}$ may be an arbitrary function.

Under similar assumptions,  \citet{idf_bivariate} suggested to estimate the noise $\varepsilon = h(Y) - g(X)$ by representing $h$ and $g$ with Multi-layer Perceptrons (MLPs) and minimizing mutual information with $X$ via gradient-based methods. However, the main drawback of their methodology is that this model can fit perfectly to any data by learning constant functions $h$ and $g$. The estimated noise will be constant and thus always independent from $X$. 

To overcome this problem \cite{pnl_bv_2020} implemented an additional auto-encoding structure in the minimization problem that enforces invertibility of the function $h$. While this circumvents the problem of constant estimation of the function $h$, there are further challenges 
that arise from minimizing dependence and subsequent testing for independence.  Indeed,  
the complexity of the function class assumed for $g$ needs to be balanced very carefully with the available sample size.  Otherwise, $g$ can be fitted perfectly such that $g(X) = h(Y)$, in which case the functional estimates cancel and the estimated noise is always independent of $X$. \cite{pnl_bv_2020} used a fixed architecture for the function classes of $g$ and $h$.  As a result, especially for small sample sizes (compared to the complexity of the function class of $g$), their method is prone to overfitting and canceling the effect of the function $h$. 
Such overfitting may then entail erroneous results in independence tests for causal structure learning. 



To avoid the noted overfitting issues, we propose the following two-stage method to learn the functional relations.    In the first stage, we leverage rank statistics to separately estimate the function $g$, without any appeal to measures of dependence between the noise and the predictor $X$.  The strictly increasing function $h$ preserves the ranks of $Y$ and thus, using rank-based methods, we can avoid estimating $h$ at this stage.  This circumvents the problem of $g$ matching $h$. In the second step, we estimate the functional relation $h$ at all observed data points to obtain the required estimates of the noise.

In order to simplify the concept and a theoretical analysis of our proposed method, we assume in the following linearity of the function $g$, i.e. $g(X) = X^T \beta_0$ for $\beta_0 \in \mathbb{R}^m$. This can also be seen as a first order Taylor approximation of an arbitrary functional relation.

\begin{remark}
\label{remark:pnl_regression_nonlinear}
Our framework and the idea of disentangling learning and testing by employing rank-based objective functions can be easily extended to the nonlinear case. By employing basis expansions, MLPs or any parametric function class in combination with the proposed rank-based scores to learn the functional relations one can trace the steps of the presented linear case. For instance, consider the basis functions $\{ b_l(\cdot) : l = 1, \dots, a_n \}$, where $a_n \to \infty$ sufficiently slowly, similar to \cite{CAM_2014}. Then we can represent the (nonlinear) function $g$ by $\sum_{l = 1}^{a_n} \alpha_l b_l(\cdot)$, where $\alpha_l \in \mathbb{R}$ for all $l=1,\dots, a_n$, and employ our proposed framework. The simulations in Section \ref{section:simulations} include an example.
\end{remark}

 We start by studying the special case of model \eqref{model:pnl_regression} under the assumption of Gaussian noise and derive a computationally efficient algorithm for estimating the functional parameters. Further, in Subsection \ref{subsection:ranks} we consider the general case without restricting the noise distribution. The main idea of our approach is to leverage rank likelihoods, however, using the full marginal rank likelihood is not computationally tractable. A common approach to circumvent calculating the full marginal rank likelihood is to employ approximate Monte Carlo methods, e.g., considered by \cite{doksum87}. However, we observed that this approach does not work well in practice for values of $\beta_0$ larger than one. Thus, in our proposed framework we employ pairwise rank likelihoods to approximate the full marginal rank likelihood (in the Gaussian case) or the rank correlation function (in the general noise case).

\subsection{Gaussian Case}
\label{subsection:rankg}

We assume that the data generating process follows the model
\begin{equation}
    \label{model:pnl_regression_linear}
    h(Y) = X^T \beta_0 + \varepsilon,
\end{equation}
with some unknown $\beta_0 \in \mathbb{R}^m$. Furthermore, we assume that the noise $\varepsilon$ is standard normal distributed and propose the following computationally fast algorithm to estimate the functional relations. The idea of this method is based on \cite{lin_tr_models_prl}.

We exploit the fact that $h$ is a strictly increasing function and therefore preserves the ranks of $\{Y_i\}_{i=1}^n$. The normality assumption yields $\varepsilon_i - \varepsilon_j \sim \mathcal{N}(0, 2)$ and we obtain
\begin{align*}
    \mathbb{P}(Y_j > Y_i &| X_j, X_i) = \mathbb{P}(h(Y_j) > h(Y_i) | X_j, X_i) \\ 
    &= \mathbb{P}(\varepsilon_i - \varepsilon_j < (X_j - X_i)^T \beta_0 | X_j, X_i) \\
    &= \Phi \left(\tfrac{(X_j - X_i)^T \beta_0}{\sqrt{2}} \right),
\end{align*}
where $\Phi$ is the cumulative distribution function of the standard normal distribution. The normalized log pairwise rank likelihood function is then given by
\begin{align}
\label{eqn:log_prl_objective}
    \ell_{prl} (\beta) := & 
    \binom{n}{2}^{-1} \sum_{i < j} \mathds{1}\{Y_j > Y_i\} \log \Phi \left(\tfrac{(X_j - X_i)^T \beta}{\sqrt{2}} \right) \nonumber \\
    &+ \mathds{1}\{Y_j \leq Y_i\} \log \Phi \left(\tfrac{(X_i - X_j)^T \beta}{\sqrt{2}} \right).
\end{align}
We estimate $\beta_0$ by maximizing $\ell_{prl}$, that is,
\begin{equation*}
    \hat{\beta}_{prl} := \underset{\beta \in \mathbb{R}^m}{ \arg\max } \; \ell_{prl}(\beta).
\end{equation*}
This defines a concave optimization problem, which leads to a computationally fast estimation routine for $\beta_0$ without precise knowledge or estimation of the function $h$.
\begin{prop}
\label{prop:prl_concave}
The log pairwise rank likelihood function $\ell_{prl}(\beta)$ defined in \eqref{eqn:log_prl_objective} is concave. Moreover, if we assume that $n > m$, then $\ell_{prl}(\beta)$ is strictly concave.
\end{prop}
The proof can be found in Appendix \ref{app:prop_concavity}.

Furthermore, the proposed estimator is consistent.

\begin{theorem} 
\label{thm:asm_prl}
As $n\to\infty$ (and in particular $n > m$), it holds that
 $$\hat{\beta}_{prl} - \beta_0 = o_P(1).$$
\end{theorem}
For a proof we refer the reader to Appendix \ref{app:theorem_asm_prl}.

In order to obtain an estimate of the noise, we estimate the transformation function $h$ in a second step. We employ the following computationally fast and consistent estimation routine proposed by \citet{rank_reg}. In his proposal, \citet{rank_reg} considers non-random covariates, however, the results are applicable to our setup conditional on the observed data $\{X_i\}_{i=1}^n.$ The method exploits the normality assumption as well as knowledge of the ranks of $\{Y_i\}_{i=1}^n$ and thus of $\{h(Y_i)\}_{i=1}^n$. 

Let $\hat{F} (z)$ be the adjusted empirical distribution function of $Z_i := h(Y_i)$, that is
\begin{equation*}
\label{eqn:adj_emp_cdf_z_i}
    \hat{F} (z) := \frac{1}{n+1} \sum_{i = 1}^n \mathbbm{1}\{Z_i \leq z\}.
\end{equation*}
\begin{remark}
Note that we only require the ranks of $\{h(Y_i)\}_{i=1}^n$ to obtain the estimate $\hat{F}(h(Y_i))$. Since $h$ is a strictly increasing function, the ranks of  $\{h(Y_i)\}_{i=1}^n$ are given by the ranks of $\{Y_i\}_{i=1}^n$.
\end{remark}
We denote the cumulative distribution function of a randomly chosen $Z_i$ by $F_{\beta} (z)$, that is 
\begin{equation}
    \label{eqn:cdf_random_z_i}
    F_{\beta} (z) := \frac{1}{n} \sum_{i=1}^n \Phi(z - X_i^T \beta).
\end{equation}
Then an estimator for the functional relation $h$ at the sample points $\{Y_i\}_{i=1}^n$ is given by
\begin{equation}
    \label{eqn:est_h_lst}
    \hat{h}_G(Y_i) := F_{\hat{\beta}_{prl}}^{-1}(\hat{F}(h(Y_i))), \quad  i=1, \dots, n.
\end{equation}
By extending this estimator to a step function on $\mathbb{R}$, and under some additional assumptions, \cite{rank_reg} show that this estimator converges to $h$ almost surely at all continuity points of the function $h$. 

\begin{remark}
In our setup, we assume $h$ is continuous. Furthermore, the additional assumptions that ensure consistency in the setting from \cite{rank_reg} are mainly smoothness and moment conditions on the distributions of $X$ and $\varepsilon$. In the considered Gaussian setting most of them are already satisfied. For a detailed list of the assumptions, we refer the reader to Appendix \ref{app:assumptions}.
\end{remark}

\subsection{General Case}
\label{subsection:ranks}
The problem of estimating the parameter $\beta_0$ without additional assumptions on the distribution of the noise in model \eqref{model:pnl_regression_linear} is extensively studied in the literature, see i.e. \cite{doksum87, Han_1987, Sherman_1993, Abrevaya_1999, Abrevaya_1999b, Abrevaya_2003, Cavangh_1998, Zhang_2013}. Without any restriction on expectation or variance of the noise the function $h$ is not unique, since it can be replaced by location or scale transformations. Thus, to ensure unique identification, we assume that there exists a known $y_0$ such that $h(y_0) = 0$ and we scale the last element of $\beta_0$ to $1$, that is, $\beta_0 = (\theta_0, 1)$.

To simplify the optimization, we employ a method introduced by \citet{smoothed_ltm_2013} that utilizes the rank-based objective function
\begin{equation*}
\label{eqn:smoothed_objective}
\begin{aligned}
    S (\beta) := \binom{n}{2}^{-1} & \sum_{i < j} \Big( \mathds{1}\{Y_j > Y_i\}\Phi \left(\sqrt{n}(X_j - X_i)^T \beta \right) \\
    &+ \mathds{1}\{Y_j \leq Y_i\} \Phi \left(\sqrt{n} (X_i - X_j)^T \beta \right) \Big).
\end{aligned}
\end{equation*}

\begin{remark}
\citet{smoothed_ltm_2013} used the assumption $\norm{\beta_0}_2 = 1$ to ensure unique identification, which is equivalent to our assumption $\beta_0 = (\theta_0, 1)$ up to rescaling.
\end{remark}
To obtain a sparse solution the authors focused on a penalized version of $S(\beta)$. However, as we are only interested in estimating the residuals and do not necessarily need sparsity, we adapted their analysis for the following simplified non-penalized estimator $\hat{\beta}:= (\hat{\theta}, 1)$, where
 \begin{equation*}
     \label{eqn:est_beta_smoothed}
     \hat{\theta} := \underset{\theta \in \mathbb{R}^{m-1}}{ \arg\max } \; S(\theta, 1),
 \end{equation*}
by setting their penalty term to zero. 

Under some additional assumptions that mainly ensure smoothness of the distributions of $X$ and $\varepsilon$, \citet{smoothed_ltm_2013} prove the existence of a local maximizer $\hat{\theta}$ of $S(\theta, 1)$ with 
\begin{equation*}
    \norm{\hat{\theta} - \theta_0}_2 = O_P\left(n^{-\frac{1}{2}}\right)
\end{equation*}
and, thus, $\hat{\beta}$ defines a consistent estimator for $\beta_0$. This estimator uses only the rank information without requiring concrete knowledge of $h$. A detailed list of the additional assumptions that ensure consistency can be found in Appendix \ref{app:assumptions}.

In a second step we estimate the function $h$ in order to subsequently obtain an estimate of the noise. We used the method introduced in \cite{chen_2002} based on the rank correlation. To simplify the complex optimization of discrete objective functions, we employ a smoothed version introduced by \cite{Zhang_2013}.

The smoothed rank correlation objective function is defined by
\begin{align*}
    Q(z, y, \hat{\beta}) := \frac{1}{n(n-1)}& \sum_{i \neq j}  [ (d_{jy} - d_{iy_0})  \\
    & \times \Phi(\sqrt{n} ((X_j - X_i)^T \hat{\beta} - z)) ],
\end{align*}
where $d_{jy} := \mathds{1}(Y_j \geq y)$ and $d_{iy_0} := \mathds{1}(Y_i \geq y_0)$. Then we define an estimator of the function $h$ at $y$ via
\begin{equation}\label{eqn:est_h_smoothed}
    \hat{h}(y) := \underset{z \in \Omega_h}{ \arg\max } \;  Q(z, y, \hat{\beta}) \; ,
\end{equation}
where $\Omega_h$ is an appropriate compact set. 

In Theorem 4.1, \citet{Zhang_2013} establish consistency of the proposed estimator for $h$ under a few assumptions, that include $\sqrt{n}$-consistency of the involved estimator $\hat{\beta}$ and strict  monotonicity of the function $h$, as well as some additional regularity assumptions. A detailed list of the additional assumptions can be found in Appendix \ref{app:assumptions}.

Without restricting the optimization space by $\Omega_h$ the problem \eqref{eqn:est_h_smoothed} is ill-posed in the sense that $\hat{h}(y) \to \infty$ for $y = \max \{Y_i\}_{i=1}^n$ and $\hat{h}(y) \to -\infty$ for $y = \min \{Y_i\}_{i=1}^n$. To circumvent the issue of choosing a proper compact set $\Omega_h$, we added an $L_2$ regularization term and optimized over $\mathbb{R}$, that is,
\begin{equation*}
    \label{eqn:est_h_smoothed_experiments}
    \hat{h}(y) := \underset{z \in \mathbb{R}}{ \arg\max } \;  \{Q(z, y, \hat{\beta}) \;  - \lambda z^2 \}.
\end{equation*}
 In the experiments we used the regularization parameter $\lambda = 10^{-3}$, which turned out to be small enough to not affect the estimated values significantly and at the same time bounded the objective function for the observed extremes.

\section{LEARNING PNL MODELS}
\label{section:learn_pnl}

By combining the previously introduced rank-based estimators of the functional relations in PNL models we obtain estimates of the stochastic error terms. Using these rank-based estimated error terms,  we propose a routine to learn the underlying causal structure by recursively identifying sink nodes via independence testing. Further, we show that our proposed routine consistently recovers a valid causal ordering under identifiability assumptions.

Suppose we observe data in form of $n$ independent copies $X_1, \dots , X_n$ from a random vector $ X := (X^{(1)}, \dots, X^{(m)})$. We assume that $X$ follows a PNL causal model, that is, the data generating process is defined by the structural equations 
\begin{equation*}
    X^{(k)} = f^{(k)}\left(g^{(k)}\left(X^{(\textbf{PA}_k)}\right) + \varepsilon^{(k)} \right), \quad k=1, \dots , m,
\end{equation*}
where $\textbf{PA}_k$, called the parents of  $X^{(k)}$, are a subset of $\{1, \dots ,m\}\setminus\{k\}$. The causal perspective stems from viewing those equations as making assignments. Each variable on the left-hand side is assigned the value specified on the right-hand side, given by the value of its parents and a stochastic error term. The causal structure inherent in such structural equations is naturally represented by a directed graph $\mathcal{G}^0$, where edges indicate which other variables each variable causally depends upon. As in related work, we assume the corresponding directed graph to be acyclic (DAG). The noise variables $\{\varepsilon^{(k)}\}_{k=1}^m$ are assumed to be mutually independent and $\varepsilon^{(k)} \indep X^{(\textbf{PA}_k)}$ for each $k=1, \dots , m$. The main ansatz for inferring the causal structure is to leverage the independence structure of the stochastic noise $\varepsilon^{(k)}$ and a correctly specified parent set, that is,  $\varepsilon^{(k)}$ is independent of all $X^{(j)}$ that precede $X^{(k)}$ in at least one true causal ordering of the underlying graph. 

We focus on inferring the causal ordering to reduce the computational burden, however, the framework can easily be adapted to infer the specific causal graph structure by pruning redundant edges. The causal ordering of a graph $\mathcal{G}^0$ is given by a permutation $\pi$ of $\{1, \dots ,m\}$, such that, if there exists a directed edge from node $\pi(i)$ to node $\pi(j)$ in the graph then $i<j$. We emphasize that the causal ordering for a given graph is not necessarily unique but each causal ordering $\pi$ corresponds to a unique, fully connected DAG $\mathcal{G}^{\pi}$, where  $\mathcal{G}^{\pi}$ has a directed edge from node $\pi(i)$ to node $\pi(j)$ if and only if $i < j$. Thus, similar to \cite{CAM_2014}, we can define the set of true causal orderings $\Pi^0$ for any DAG $\mathcal{G}^0$ as the set of all causal orderings $\pi$ that correspond to fully connected DAGs $\mathcal{G}^{\pi}$ which contain $\mathcal{G}^0$ as a sub-graph, that is
\begin{equation*}
    \Pi^0 := \{\pi \text{ : } \mathcal{G}^{\pi} \text{ is a super-graph of } \mathcal{G}^0 \}.
\end{equation*}

\begin{remark}
In general $\Pi^0$ contains more than one element and all elements correspond to valid causal orderings of the DAG $\mathcal{G}^0$ (e.g. in the extreme case of an empty graph,  all permutations are true causal orderings).
\end{remark}
In order to apply the previously introduced rank-based regression methods, we assume that all functions $f^{(k)}$ in the data generating PNL causal model are continuous and strictly increasing, and thus, we can define their inverse via $h^{(k)}:=(f^{(k)})^{-1}$. Further, we assume that all $g^{(k)}$ are linear and the distribution of every stochastic error $\varepsilon^{(k)}$ is assumed to be continuous. We emphasize again, our method is applicable to nonlinear functional relations by means of basis expansions. 

Put together, each structural equation, i.e. each cause-effect relation, corresponds to a PNL regression model \eqref{model:pnl_regression} as introduced in the previous section. That is, the data generating process follows
\begin{equation*}
\label{model:pnl}
    h^{(k)}\left(X^{(k)}\right) = \left(X^{(\textbf{ND}_k)}\right)^T \beta^{(k)} + \varepsilon^{(k)}, \quad  k=1, \dots , m,
\end{equation*}
where the non-descendants $\textbf{ND}_k$ are given by all nodes that precede node $k$ in at least one true causal ordering and $\varepsilon^{(k)}$ is independent of $X^{(\textbf{ND}_k)}$. Note that the entries in $\beta^{(k)}$ which do not correspond to parents of node $k$ are simply zero. To leverage the independence $\varepsilon^{(k)} \indep X^{(\textbf{ND}_k)}$, we define 
\begin{equation*}
    X^{(-k)} := (X^{(1)}, \dots , X^{(k-1)}, X^{(k+1)}, \dots ,  X^{(m)}).
\end{equation*}
For every sink node $k$ in the graph, we have $\textbf{ND}_k = \{1, \dots, m \}\setminus \{ k \}$, and, thus, for every sink node $k$ the noise $\varepsilon^{(k)}$ is independent of $X^{(-k)}$. Moreover, if node $k$ is not a sink node in the graph, then the noise $\varepsilon^{(k)}$ is not independent of $X^{(-k)}$, since $X^{(-k)}$ contains at least one child of $k$. Thus, we can recursively identify a sink node using the HSIC \citep{HSIC_2005} measure of independence between the estimated noise $\hat{\varepsilon}^{(k)}$ and the remaining nodes $X^{(-k)}$.

We propose the following routine to learn one of the valid causal orderings of the graph. First, we utilize the rank-based estimators $\hat{h}^{(k)}$ and $\hat{\beta}^{(k)}$, introduced in the previous section, and estimate the noise via
\begin{equation*}
    \hat{\varepsilon}_j^{(k)} := \hat{h}^{(k)}\left(X_j^{(k)}\right) - \left(X_j^{(-k)}\right)^T \hat{\beta}^{(k)}, \quad j=1, \dots , n.
\end{equation*}
We repeat this noise estimation for all remaining nodes $k \in \{1, \dots , m\}$ and  subsequently calculate the HSIC test statistic between the estimated noises and the observed values of the remaining nodes $X^{(-k)}$, that is
\begin{equation*}
\label{eqn:t_k_definition}
    t_k := HSIC(\{ X_j^{(-k)}, \hat{\varepsilon}_j^{(k)} \}_{j=1}^n),  \quad k=1, \dots , m.
\end{equation*}
We determine the node which leads to the minimal test statistic as a sink node, that is, our proposed sink node estimator is defined by
\begin{equation*}
\label{eqn:sink_node_est_equality}
    \hat{\pi}(m) := \underset{k}{\arg\min} \; \{ t_k \} \;.
\end{equation*}
In the next step we remove $\hat{\pi}(m)$ from the set $\{1, \dots , m\}$ and repeat the sink node identification procedure to estimate $\hat{\pi}(m-1)$. Thus, recursively we obtain an estimate for the causal ordering 
\begin{equation*}
\label{eqn:sink_est_order}
    \hat{\pi} = (\hat{\pi}(1), \dots, \hat{\pi}(m)).
\end{equation*}
In the following Theorem we prove consistency of our proposed estimation routine for the causal ordering. It is clear that if at any step our method fails to correctly identify a remaining sink node, then it fails to estimate a valid causal ordering. Thus, we must require sink node identifiability from the joint distribution in order to ensure consistency of the estimated causal order. The following assumption \textbf{(A)} formalizes this intuition. 

\textit{\textbf{(A)} For each $k \in [1, m]$ and $A \subset [1, m] \setminus \{ k \}$ that contains at least one child of $X^{(k)}$ as well as for all strictly increasing, continuous functions $h: \mathbb{R} \to \mathbb{R}$ and for all $\beta \in \mathbb{R}^{|A|}$, there exists a constant $\xi > 0$, such that 
\begin{equation*}
    HSIC \left(\mathbb{P}^{N, X^{(A)}} \right) > \xi,
\end{equation*}
where
\begin{equation*}
    N := h\left(X^{(k)}\right) - \left(X^{(A)}\right)^T \beta.
\end{equation*}
}

\begin{theorem}
\label{thm:causal_order_consistency}
Under assumption \textbf{(A)} and consistency of the employed estimators $\hat{h}^{(k)}$ and $\hat{\beta}^{(k)}$ we have
\begin{equation*}
    \mathbb{P}(\hat{\pi} \in \Pi^0) \to 1 \text{ as } n \to \infty.
\end{equation*}
\end{theorem}

The proof can be found in Appendix \ref{app:proof_of_consistency}.

\begin{remark}
Location and scale transformations of the noise variables can be matched by transformations in the functions $h^{(k)}$, however, these transformations do not change the dependence structure and thus without loss of generality, we can assume that the location and scale assumptions in Section \ref{section:pnl_regression} are satisfied.

\end{remark}

\section{SIMULATIONS}
\label{section:simulations}
In this section we present the results of a simulation study with the aim to compare the performance of our algorithm to existing causal learning methods on various simulated data sets and validate the consistency results experimentally. If necessary in the specific application, our method can be easily extended to infer the full causal graph structure, however, this vastly increases the computation time, similar to the related existing methods. Since the main differences of the algorithms already come into play during the causal order estimation procedure, we compared the performance on this task alone. Our experiments were designed as follows. We randomly sampled a causal structure with $m$ ($m = 4$ and $m = 7$) nodes from Erdős–Rényi directed acyclic graphs with edge probability $\frac{2}{m-1}$. Thus, the expected total number of edges in the graph is $m$. We generated data according to the corresponding post-nonlinear model using the following structural relations
\begin{equation*}
    X^{(k)} = \Big(\sum_{j \in \textbf{PA}_k} \beta_{1j} X^{(j)} + \beta_{2j} (X^{(j)})^2    + \varepsilon_k \Big)^{1/3}
\end{equation*}
 for $k=1, \dots , m,$ with different noise distributions (standard Normal $\mathcal{N}(0, 1)$, standard Gumbel $Gumbel(0, 1)$ or standard Logistic $Logistic(0,1)$). Further, $\beta_{1j}$ and $\beta_{2j}$ are sampled from a uniform distribution with either a low range from $-10$ to $10$, representing a weak signal setting (low signal-to-noise ratio (SNR)), or a higher range from $-100$ to $100$, representing a strong signal setting (high SNR).

\begin{figure*}[h!]
    \centering
    \includegraphics[width=1.0\textwidth]{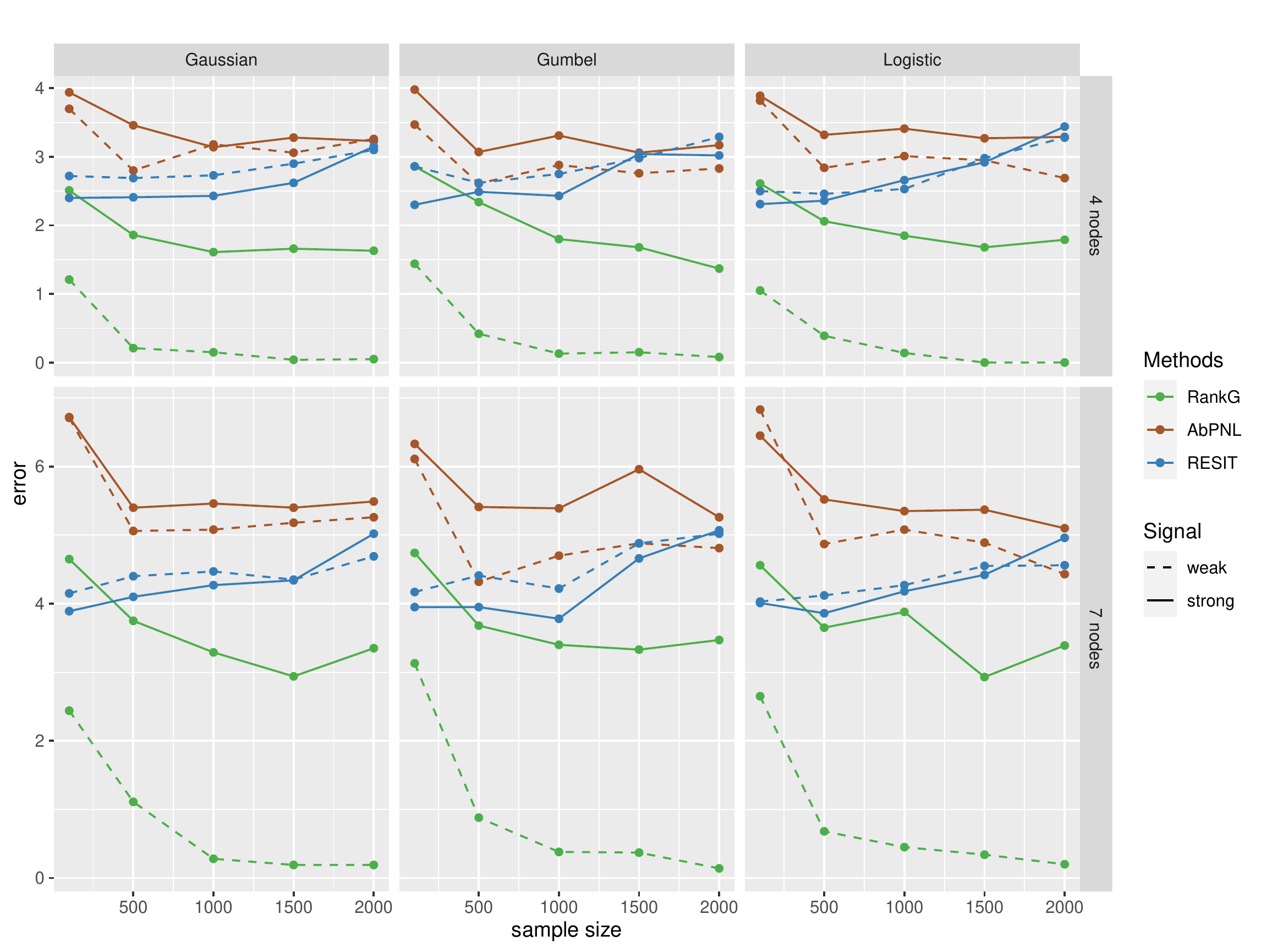}
    \caption{Performance of \texttt{RankG}, \texttt{AbPNL} and \texttt{RESIT} causal order estimation methods in different dimensions (4,7) against sample size.}
    \label{fig:res_prlg}
\end{figure*}

\begin{figure*}[!h]
    \centering
    \includegraphics[width=1.0\textwidth]{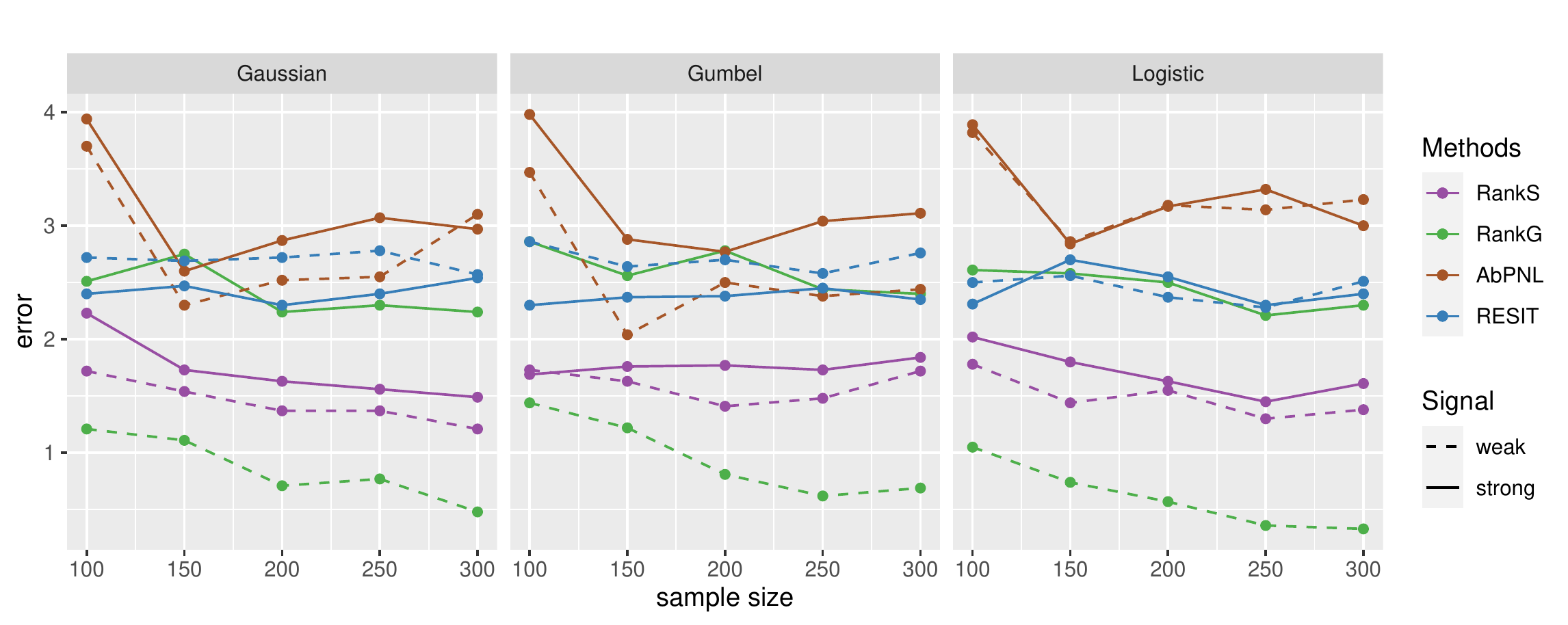}
    \caption{Performance of \texttt{RankS}, \texttt{RankG}, \texttt{AbPNL} and \texttt{RESIT} causal order estimation methods in dimension 4 against sample size.}
    \label{fig:res_smoothed}
\end{figure*}

\begin{figure*}[!htbp]
    \centering
    \includegraphics[width=1.0\textwidth]{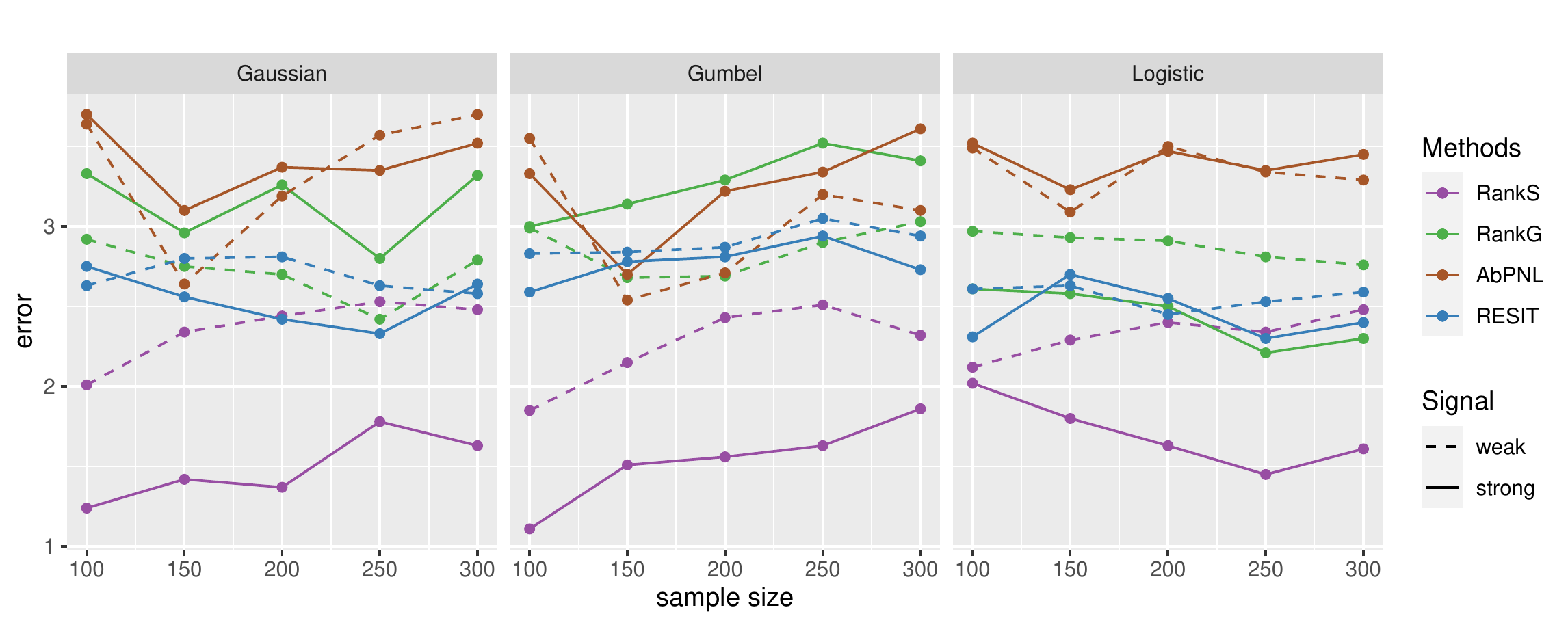}
    \caption{Performance of \texttt{RankS}, \texttt{RankG}, \texttt{AbPNL} and \texttt{RESIT} causal order estimation methods in dimension 4 against sample size, where function $g$ is 4 degree polynomial.}
    \label{fig:res_smoothed_4g}
\end{figure*}

\begin{remark}
We highlight that the inner function $g^{(k)}$ is quadratic in the parents of $X^{(k)}$. In the experiments, we used polynomial basis expansions of order two to linearly model the inner function by specifying not only parents but also their squares.
\end{remark}

Using this process, we generated $100$ independent data sets and estimated the causal ordering. We compared our results with the classical \texttt{RESIT} method for additive noise models \citep{ANM_2011} and the \texttt{AbPNL} method for post-nonlinear models \citep{pnl_mult_2022} restricted to the respective causal order estimation parts. The causal ordering of a given DAG is not necessarily unique, thus, as a measure of performance for an estimated causal ordering $\hat{\pi}$ we report the number of directed edges  $\hat{\pi}(i) \to \hat{\pi}(j)$ in the true graph $\mathcal{G}$ with $j<i$, that is 
\begin{equation*}
\label{eqn:causal_order_estimation_error}
    \# \{ (i, j) : \hat{\pi}(i) \to \hat{\pi}(j) \in \mathcal{G} \space \text{ and } j < i  \}.
\end{equation*}
This measure equals zero when $ \hat{\pi}$ is a valid causal ordering and achieves its maximum, the number of edges in $\mathcal{G}$, when $ \hat{\pi}$ is a reversed causal ordering.

Figure \ref{fig:res_prlg} shows the performance of the Gaussian method introduced in Subsection \ref{subsection:rankg}, named \texttt{RankG}, compared to \texttt{RESIT} and \texttt{AbPNL} on 4- and 7-dimensional causal graphs in settings with Gaussian, Gumbel (standard extreme value distribution) and Logistic noise. Dashed lines indicate the weak signal setting and solid lines depict the strong signal setting. We plot the mean of our performance measure, the number of wrongly oriented edges in the fully connected DAG corresponding to the estimated causal ordering, over all 100 data sets against the sample size. 

Our proposed \texttt{RankG} method outperforms the competition in almost all considered settings, especially in the low SNR setting. We emphasize that it might seem counterintuitive that the \texttt{RankG} method performs better in a low SNR setting than in a high SNR setting, however, the noise drives the identification in the rank-based learning procedure in PNL models. Thus, higher noise in comparison to the signal strengths induces more changes in the ranks that propagate through the graph, and thus, better performance of the rank-based estimation methods. 

Furthermore, in the low SNR setting our computational results support the theoretical consistency results and our method seems to recover a valid causal ordering even in moderate sample sizes.

The results in Figure \ref{fig:res_prlg} display that even under noise misspecification, that is, in the Gumbel and Logistic noise cases, the \texttt{RankG} method performs best. This might indicate some robustness of our proposed method for causal order estimation, even though we do not recover the true noise under misspecification (see Figure  \ref{fig:res_noise_est}).

We conducted similar experiments to compare the performance of our proposed general method introduced in Subsection \ref{subsection:ranks}, named \texttt{RankS}, however with lower sample sizes for computational reasons. Figure \ref{fig:res_smoothed} shows the results of the different competitors \texttt{RankS}, \texttt{RankG}, \texttt{RESIT} and \texttt{AbPNL} for 4-dimensional graphs with Gaussian, Gumbel and Logistic noise. 
Our proposed method \texttt{RankS} performs best in all considered sample sizes and all noise cases, except for the weak signal settings where \texttt{RankG} performs better. This might stem from the fact that the pairwise rank likelihood used in \texttt{RankG} better approximates the marginal rank likelihood.

In additional experiments, we investigated the behavior of the introduced methods for more complex functional relations $g$, namely a polynomial of degree 4. Similar to the experiments before, we sampled data from Erdős–Rényi DAGs, however, with the structural relations
\begin{equation*}
\begin{aligned}
    X^{(k)} =  \Big( & \sum_{j \in \textbf{PA}_k} \beta_{1j} X^{(j)} + \beta_{2j} (X^{(j)})^2 \\
    &+ \beta_{3j} (X^{(j)})^3 + \beta_{4j} (X^{(j)})^4 + \varepsilon_k \Big)^{1/3}
\end{aligned}
\end{equation*}
for $k=1, \dots , m,$ with different noise distributions and $\beta_{1j}, \beta_{2j}, \beta_{3j}$ and $ \beta_{4j}$ sampled with low or high SNR. Figure \ref{fig:res_smoothed_4g} shows the resulting mean performance measures for the different methods \texttt{RankS}, \texttt{RankG}, \texttt{RESIT} and \texttt{AbPNL} over 100 data sets. As expected, the task becomes more challenging by increasing the complexity of the functional parameters, since it is difficult to estimate the functional relations in the first place. However, even in the considered low sample sizes, our proposed  \texttt{RankS} method seems to detect some causal structure and outperform the competition.

Further, we analysed the behaviour of the used functional estimators introduced in Section \ref{section:pnl_regression} by performing the following experiments. We generated 100 independent data sets of sample size 500 according to model \eqref{model:pnl_regression_linear} with a 3-dimensional predictor $X$, standard Gaussian or Gumbel noise, cubic function $h(y) = y^3$ and linear parameters $\beta_0 = (10, 5, 1)$. Then we estimated the functional parameter $\beta_0$ and the function $h$ with the introduced rank-based methods. 

Figure \ref{fig:res_h_est} shows the estimation of the function $h$ for one representative result across the 100 replications. The red lines indicate the true value of the function $h$, while the black dots indicate the pointwise estimates. We notice that the estimation of the function $h$ with the rank-based method that relies on the normality assumption (used in \texttt{RankG}) fails to correctly estimate the functional relation at extremes under Gumbel noise. This is due to the misspecified tail probability structure. In contrast, the general estimation method employed in \texttt{RankS} is not influenced by the specific underlying noise distribution. However, in the Gaussian noise case we notice small estimation bias, which can be regulated with the hyperparameter $\lambda$ in the estimation procedure. Further, from the results across all 100 data sets, we noticed that the variance of the functional estimate across the data sets is higher using the general estimation methods in \texttt{RankS}.

Figure \ref{fig:res_beta_est} shows the estimation of the first two entries of $\beta_0$. Recall that in the general \texttt{RankS} method we fixed the last entry in our estimation of $\beta_0$. The box plots show the estimated values of $\beta_0$ across the 100 data sets and red dots indicate the true values. We see again that \texttt{RankG} estimates the parameter $\beta_0$ with a bias in the misspecified Gumbel noise setting, similar to the estimation of the function $h$. However, in the Gaussian setting the \texttt{RankG} method estimation of $\beta_0$ has a lower variance than \texttt{RankS} and in all other cases the median estimate corresponds to the true value.

\begin{figure}[!h]
    \centering
    \includegraphics[width=0.5\textwidth]{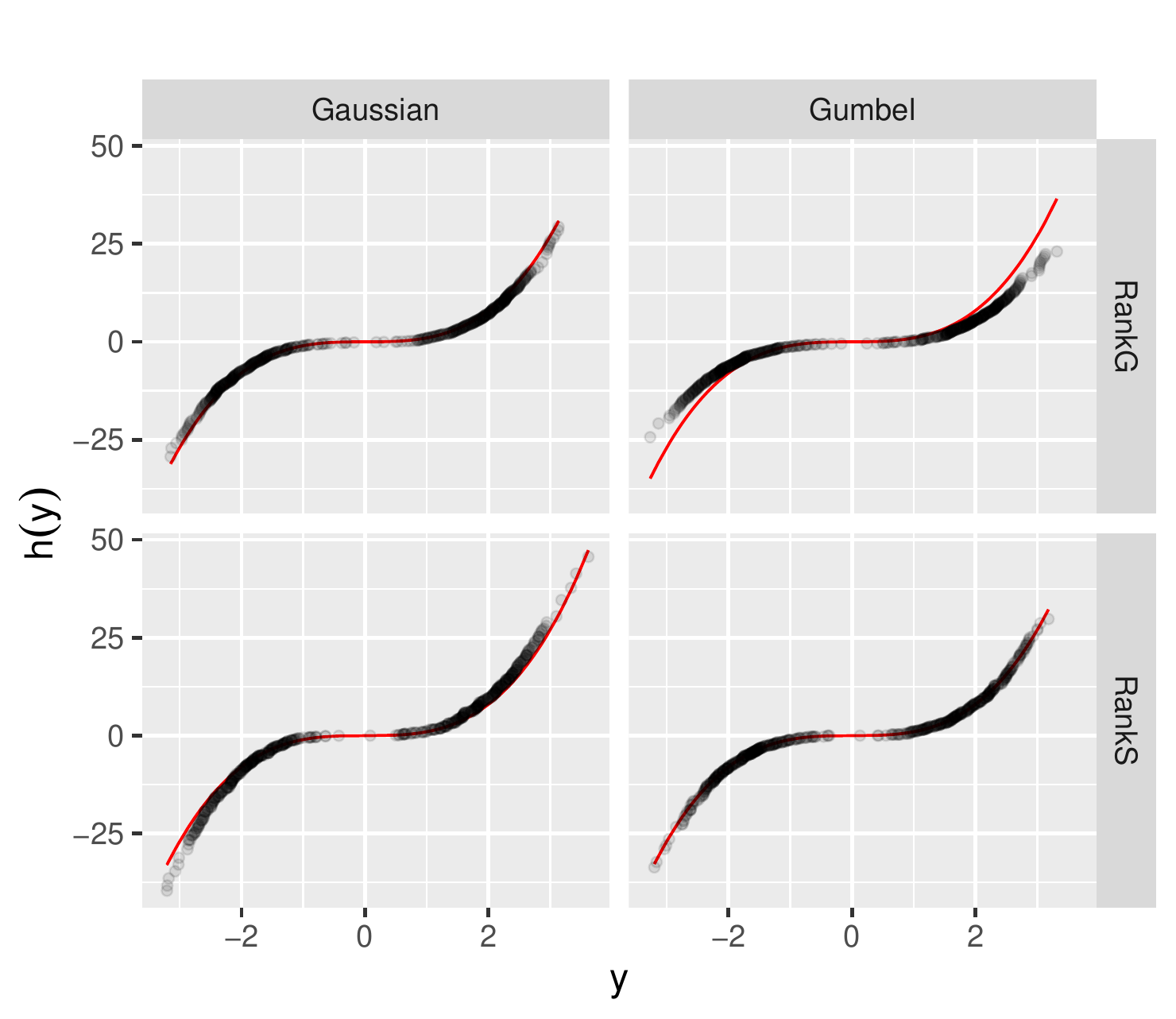}
    \caption{Estimation of the function $h$.}
    \label{fig:res_h_est}
\end{figure}

\begin{figure}[!h]
    \centering
    \includegraphics[width=0.5\textwidth]{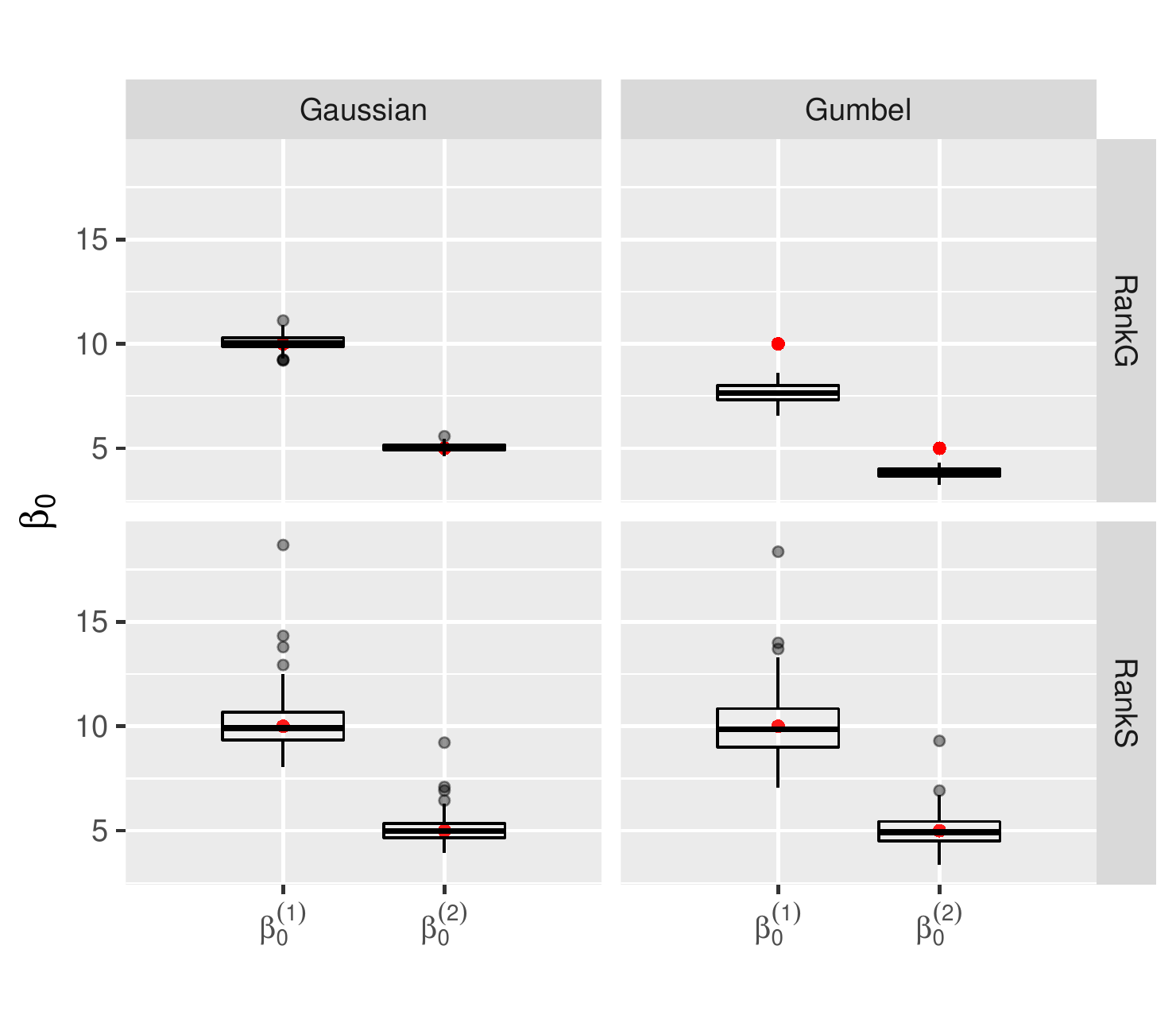}
    \caption{Estimation of $\beta_0$.}
    \label{fig:res_beta_est}
\end{figure}

\begin{figure}[!h]
    \centering
    \includegraphics[width=0.5\textwidth]{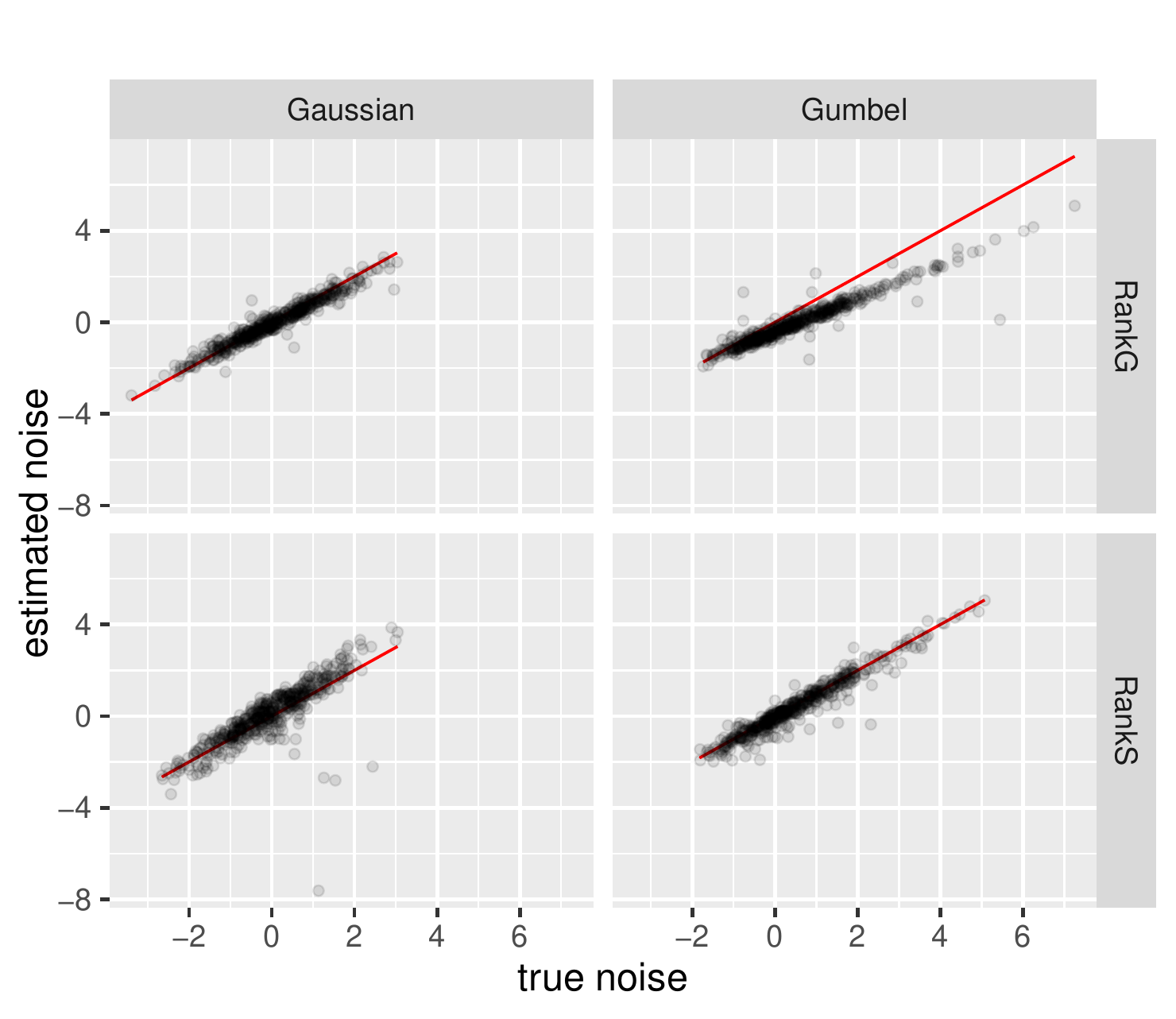}
    \caption{Estimation of the noise $\varepsilon$.}
    \label{fig:res_noise_est}
\end{figure}

Figure \ref{fig:res_noise_est} shows the estimated noise by combining both estimation results. We plot the estimated noise against the true values for one representative data set (similar to Figure \ref{fig:res_h_est}). Red lines correspond to a perfect estimation. The \texttt{RankG} method inherits the behaviour from both estimation parts and fails to correctly estimate the extreme noise cases. However, it outperforms the \texttt{RankS} method in the Gaussian setting.

\section{CONCLUSION}
\label{section:conclusion}
We proposed a new routine for causal discovery in multivariate post-nonlinear structural equation models. Our method disentangles the two tasks of estimating the functional relations and learning the causal structure by employing rank-based methods for the first task. Thus, our proposed routine is less susceptible to overfitting issues exhibited by the existing methods that rely on
minimizing dependence and subsequently testing for independence.

We introduced PNL rank regression methods to learn the functional relations in PNL models and subsequently estimate the residuals. As a special case, we first considered Gaussian noise and used pairwise rank likelihoods in a computationally fast algorithm, whereas, for the general noise case, we employed a smoothed version of rank correlations to obtain estimates of the functional relations. While our presentation focused on linearity in the inner functional relation to simplify  the theoretical analysis, the framework includes nonlinear relations by means of basis expansions. Employing the introduced estimators of the functional relations, we proposed a causal learning routine to recursively identify sink nodes based on independence tests with the estimated residuals. Further, we prove consistent causal order recovery of our proposed routine under identifiability assumptions and consistency of the employed functional estimators.

We validated our theoretical findings in a simulation study that showed that our proposed routine outperforms the competition and is able to recover a valid causal ordering even in moderate sample sizes.

\subsubsection*{Acknowledgements}
This project has received funding from the European Research Council (ERC) under the European Union’s Horizon 2020 research and innovation programme (grant agreement No. 83818), the German Federal Ministry of Education and Research, and the Bavarian State Ministry for Science and the Arts. The authors of this work take full responsibility for its content.

\bibliography{refs.bib}

\appendix
\onecolumn

\section{ADDITIONAL ASSUMPTIONS}
\label{app:assumptions}
In this section, we list all additional assumptions that are mentioned in the main paper in detail. 

\citet{rank_reg} employs the following 4 (additional) assumptions to ensure the consistency of $\hat{h}_G$ defined in \eqref{eqn:est_h_lst}. Thus, the following assumptions together with the modelling assumptions and assumption \textbf{(A)} in the main paper ensure the consistency of the proposed \texttt{RankG} method. 

\begin{enumerate}[label=\textbf{AG\arabic*}]
    \item Let $G_n(x) := \frac{1}{n} \sum_{j=1}^n \mathds{1} \{X_j \leq x \} $. Assume $Z_n$ has distribution $G_n$, then $|Z_n|^t$ is uniformly integrable for some $t$ which is specified in the next assumption (at least $t>4$). 
    \item For $F_{\beta}(z)$ defined in \eqref{eqn:cdf_random_z_i}, there exists finite $K$ such that 
    \begin{equation*}
        F_{\beta}^{-1}(u) + u(1-u)(F_{\beta}^{-1}(u))' \leq \frac{K}{((u(1-u)))^{\alpha}}, \quad \forall \ u \in (0, 1),
    \end{equation*}
    for all $n$ and $\beta \in B$, where $B$ is a neighborhood of $\beta_0$ and $\alpha + t^{-1} < \frac{1}{2}$.
    
    \item The function $Z(\beta):=\mathbb{E}[X| X^T (\beta -\beta_0) = 0]$ is $L_2$ continuous as $\beta \to \beta_0$.
    \item The following inequality holds
    \begin{equation*}
        \mathbb{E} [X(X - Z(\beta_0)] > 0.
    \end{equation*}
    
\end{enumerate}
The main assumptions above essentially correspond to moment conditions on the distribution of $X$. 

Following \cite{smoothed_ltm_2013} (\textbf{AS1-AS4}) and \cite{Zhang_2013} (\textbf{AS5-AS9}), we list the additional assumptions that are required to ensure consistency of the estimators in Section \ref{subsection:ranks}. Thus, in combination with the modelling assumptions and assumption $\textbf{(A)}$ in the main paper, the \texttt{RankS} method is consistent.

\begin{enumerate}[label=\textbf{AS\arabic*}]
    \item Let $g$ be the density function of $(X_j - X_i)^T \beta_0$ and $F$ the distribution function of $\varepsilon_j - \varepsilon_i$. Define the functions $\Gamma(s) := \mathbb{E} [(X_j-X_i)^T | (X_j - X_i)^ \beta_0 = s] $ and $\Omega(s) := F(s)\Gamma(s)g(s)$. Then $\Omega'(0)$ is nonsingular.
    \item The density $g$ is positive with a continuous second derivative on its corresponding compact support.
    \item The function $F$ has a continuous second derivative on its corresponding support.
    \item The random variable $X$ is bounded with a compact support. 

    \item The true parameter $\theta_0$ is an interior point of a compact subset $\Theta \subset \mathbb{R}^{m-1}$.
    \item The support of $X$ is not contained in a linear subspace of $\mathbb{R}^m$. Moreover, conditional on the first $m-1$ components of $X$, the last component of $X$ has a density function with respect to the Lebesgue measure. 
    \item Define 
    \begin{equation*}
        \tau (y, x, \theta) := \mathbb{E} [\mathds{1} \{ y > Y\} \mathds{1} \{ (x-X)^T (\theta^T, 1)^T > 0\} + \mathds{1} \{ y < Y\} \mathds{1} \{ (x-X)^T (\theta^T, 1)^T < 0\}]
    \end{equation*}
    and let
    \begin{equation*}
        |\nabla_m| \tau (y, x, \theta) = \sum_{i_1, \dots, i_m} \left|\frac{\partial^m \tau (y, x, \theta)}{\partial \theta_{i_1} \dots \partial \theta_{i_m}}\right|.
    \end{equation*}
    There exists a neighborhood $\mathcal{N}$ of $\theta_0$ such that for each pair of $(y, x)$ in the support of $(Y, X)$ the following hold 
    \begin{itemize}
        \item The second derivatives of $\tau (y, x, \theta)$ with respect to $\theta$ exist in $\mathcal{N}$.
        \item There exists an integrable function $M(y, x)$ such that for all $\theta$ in $\mathcal{N}$
        \begin{equation*}
            \norm{\nabla_2 \tau (y, x, \theta) - \nabla_2 \tau (y, x, \theta_0)} \leq M(y, x) |\theta - \theta_0|.
        \end{equation*}
        \item $\mathbb{E} [(|\nabla_1| \tau (Y, X, \theta_0))^2] < \infty$.
        \item $\mathbb{E} [|\nabla_2| \tau (Y, X, \theta_0)] <  \infty$.
        \item The matrix $\mathbb{E} [ \nabla_2 \tau (Y, X, \theta_0)]$ is strictly negative definite.
    \end{itemize}
    \item There exists $\epsilon^*>0$ and  $y_1, y_2$ in the support of $Y$ such that $[h(y_1-\epsilon^*), h(y_2 + \epsilon^*)]$ is contained in a compact interval.
    \item For $\omega_1 = (x^1, y^1)$, $\omega_2 = (x^2, y^2)$ and $W = (X, Y)$ we define
    \begin{equation*}
        f^z(\omega_1, \omega_2, z, y, \beta) := (\mathds{1} \{ y^1 \geq y \} - \mathds{1} \{ y^2 \geq y_0 \}) \mathds{1} \{ (x^1 - x^2)^T \beta \geq z \}.
    \end{equation*}
    Further, we define
    \begin{equation*}
        \tau(\omega, z, y, \beta) := \mathbb{E}[f^z(\omega, W, z, y, \beta) + f^z(W, \omega, z, y, \beta)],
    \end{equation*}
    then
    \begin{equation*}
        V^z(y) := \mathbb{E} \left[ \frac{\partial^2 \tau(W, h(y), y, \beta_0)}{\partial z^2} \right]
    \end{equation*}
    is negative for each $y \in [y_1, y_2]$ and uniformly bounded away from 0. 
\end{enumerate}
\citet{Zhang_2013} show uniform consistency on the interval $[y_1, y_2]$ defined in the assumptions \textbf{AS8-AS9}.

\section{PROOF OF PROPOSITION \ref{prop:prl_concave}}
\label{app:prop_concavity}
To prove Proposition \ref{prop:prl_concave} we employ the following two Lemmas.

\begin{lemma}
\label{lem:gaussian_func_negative}
For all $z \in \mathbb{R}$ we have
\begin{equation*}
    \phi'(z) \Phi(z) - (\phi(z))^2 < 0,
\end{equation*} 
where $\Phi$ and $\phi$ denote the CDF and PDF of the standard normal distribution.
\end{lemma}

\begin{proof}
With $h(z) := \phi'(z) \Phi(z) - (\phi(z))^2$, we show that $h(z) < 0$. Substituting the derivative of $\phi$ we have
\begin{equation*}
    h(z) = -z \phi(z) \Phi(z) - (\phi(z))^2 = \phi(z)(-z \Phi(z) - \phi(z)).
\end{equation*}
Since $\phi(z) > 0$ for all $z\in \mathbb{R}$ it remains to show that $g(z) := -z \Phi(z) - \phi(z) < 0$. For $z \ge 0$ this is clear. Thus we consider the case $z < 0$. We have for the derivative of $g$
\begin{equation*}
    g'(z) = -\Phi(z) -z \phi(z) + z \phi(z) = -\Phi(z) < 0.
\end{equation*}
Therefore, $g$ is a strictly decreasing function. The limit of $g(z)$ for $z \to - \infty$ is given by
\begin{align*}
    \lim_{z \to - \infty} g(z) &= \lim_{z \to - \infty} (-z \Phi(z) - \phi(z))  = -\lim_{z \to - \infty} \frac{\Phi(z)}{\frac{1}{z}} \\
    &= \lim_{z \to - \infty} \frac{\phi(z)}{\frac{1}{z^2}} = \lim_{z \to - \infty} \frac{1}{\sqrt{2 \pi}} \frac{z^2}{e^{z^2/2}} = 0,
\end{align*}
and the claim follows.
\end{proof}

\begin{lemma}
\label{lem:concave_log_gaussian_func}
The function 
$$
f(x) := \log \Phi(c^T x), \quad x \in \mathbb{R}^m
$$
is concave, where $\Phi$ is the CDF of the standard normal distribution and $c \in \mathbb{R}^m$ is a nonzero constant. Moreover, $v^T \nabla^2 f(x) v = 0$ for a vector $v \in \mathbb{R}^m$ and Hessian matrix $\nabla^2 f(x)$ if and only if $v^T c = 0$.
\end{lemma}

\begin{proof}
The function $f$ is twice differentiable, and thus for the first part of the Lemma it is enough to show that the Hessian of $f(x)$ is negative semi-definite. The gradient of $f$ is given by
\begin{equation*}
    \nabla f(x) = \frac{\phi(c^Tx)}{\Phi(c^Tx)} c,
\end{equation*}
where $\phi$ is the PDF of the standard normal distribution. Thus, the Hessian is
\begin{equation*}
    \nabla^2 f(x) = \frac{\phi'(c^Tx) \Phi(c^Tx) - (\phi(c^Tx))^2}{(\Phi(c^Tx))^2} \cdot c c^T.
\end{equation*}
So, for any $v \in \mathbb{R}^m$ we have 
\begin{align*}
    v^T \nabla^2 f(x) v &= v^T  \frac{\phi'(c^Tx) \Phi(c^Tx) - (\phi(c^Tx))^2}{(\Phi(c^Tx))^2} \cdot c c^T v \\
    &=   \frac{\phi'(c^Tx) \Phi(c^Tx) - (\phi(c^Tx))^2}{(\Phi(c^Tx))^2} \cdot v^T c c^T v \\
    &= \frac{\phi'(c^Tx) \Phi(c^Tx) - (\phi(c^Tx))^2}{(\Phi(c^Tx))^2} \cdot (v^T c)^2 \le 0,
\end{align*}
where the last step follows from Lemma \ref{lem:gaussian_func_negative} and the fact that $(v^T c)^2 \ge 0$. 

Moreover, $v^T \nabla^2 f(x) v = 0$ if and only if $v^T c = 0$, which completes the proof.
\end{proof}
Employing the Lemmas we can prove Proposition \ref{prop:prl_concave}.

\begin{proof}
From \eqref{eqn:log_prl_objective} we have 
\begin{align*}
    \ell_{prl}(\beta) = &\binom{n}{2}^{-1} \sum_{i < j} \mathds{1}(Y_j > Y_i) \log \Phi \left(\frac{(X_j - X_i)^T \beta}{\sqrt{2}}  \right) \\
    &+ \mathds{1}(Y_j \le Y_i) \log \Phi \left(\frac{(X_i - X_j)^T \beta}{\sqrt{2}}  \right).
\end{align*}
Thus, with Lemma \ref{lem:concave_log_gaussian_func} we know that $\ell_{prl}(\beta)$ is a sum of concave functions. Since sums preserve the concavity $\ell_{prl}(\beta)$ is concave. 

We show strict concavity by contradiction and thus assume that $\ell_{prl}(\beta)$ is not strictly concave. This implies that there exists a vector $v$ such that $v^T \nabla^2 \ell_{prl}(\beta) v = 0$ for the Hessian matrix $\nabla^2 \ell_{prl}(\beta)$ of $\ell_{prl}(\beta)$. The Hessian operator is linear, thus, $\nabla^2 \ell_{prl}(\beta)$ is a sum of Hessians, that is
\begin{align*}
    \nabla^2 \ell_{prl}(\beta) &= \binom{n}{2}^{-1} \sum_{i < j} \mathds{1}(Y_j > Y_i) \nabla^2 \log \Phi \left(\frac{(X_j - X_i)^T \beta}{\sqrt{2}}  \right) \\
    &+ \mathds{1}(Y_j \le Y_i) \nabla^2 \log \Phi \left(\frac{(X_i - X_j)^T \beta}{\sqrt{2}}  \right).
\end{align*}
Lemma \ref{lem:concave_log_gaussian_func} gives that 
$v^T \nabla^2 \log \Phi \left(\frac{(X_j - X_i)^T \beta}{\sqrt{2}}  \right)  v = 0$ 
if and only if 
$v^T (X_j - X_i) = 0$. Since for each $i, j$, either one of 
$\mathds{1}(Y_j > Y_i)$ or $\mathds{1}(Y_j \le Y_i)$ is 1, $v^T \nabla^2 \ell_{prl}(\beta) v = 0$ 
implies that $v^T (X_j - X_i) = 0$ for all $i$ and $j$. Therefore, $X_j^T v=c$ for a constant $c \in \mathbb{R}$ for all $j$.
Let $X^{(i)}=(X^{(i)}_1, \dots, X^{(i)}_n)$ be the sample vector of the $i-$th component of $X$ and define the matrix $\X := [\mathbf{1}, X^{(1)}, \dots, X^{(m)}] \in \mathbb{R}^{n \times m+1}$. Then $\X$ does not have full column rank, i.e. taking $u = (-c, v^T)^T$ implies $\X u = \mathbf{0}$. 

However, if we take any arbitrary square sub-matrix in $\X$ and compute the determinant, we obtain a non-zero polynomial of some $X^{(1)}, \dots, X^{(m)}$. The Lemma in \citep{okamoto1973} states that such a polynomial is zero only on the Lebesgue measure zero. Therefore, the rank of $\X$ is $\min\{n, m+1 \} = m+1$, which contradicts the equality $\X u = 0$ and, thus, completes the proof.
\end{proof}

\section{PROOF OF THEOREM \ref{thm:asm_prl}}
\label{app:theorem_asm_prl}

To keep the formulas readable, we define $U_{ij} := \frac{X_i - X_j}{\sqrt{2}}$. This gives 
\begin{equation*}
    \ell_{prl}(\beta) = \binom{n}{2}^{-1} \sum_{i < j} \mathds{1}(Y_j > Y_i) \log \Phi \left(U_{ji}^T \beta \right) 
    + \mathds{1}(Y_j \le Y_i) \log \Phi \left(U_{ij}^T \beta  \right).
\end{equation*}
For the proof, we use the Taylor expansion of $ \ell_{prl}(\hat{\beta}_{PRL})$ around $\beta_0$ and use properties of the gradient of $\ell_{prl}(\beta)$ at $\beta_0$, which are established in the following Lemmas.

\begin{lemma}
\label{lemma:grad_prl_lik_zero_as}
The gradient $\nabla_{\beta} \ell_{prl}(\beta_0)$ converges to zero almost surely, that is
\begin{equation*}
    \nabla_{\beta} \ell_{prl}(\beta_0) \overset{a.s.}{\to} 0.
\end{equation*}
\end{lemma}

\begin{proof}
The gradient
\begin{align*}
    \nabla_{\beta} \ell_{prl}(\beta_0) &= \binom{n}{2}^{-1} \sum_{i < j} \mathds{1}(Y_j > Y_i) \frac{\phi \left( U_{ji}^T \beta_0 \right) }{\Phi \left(U_{ji}^T \beta_0 \right) } \cdot U_{ji} 
    + \mathds{1}(Y_j \le Y_i) \frac{\phi \left(U_{ij}^T \beta_0 \right)}{\Phi \left(U_{ij}^T \beta_0 \right)} \cdot U_{ij},
\end{align*}
is a U-statistic with kernel
\begin{equation*}
    \psi((X_1, Y_1), (X_2, Y_2)) := \mathds{1}(Y_2 > Y_1) \frac{\phi \left(U_{21}^T \beta_0  \right) }{\Phi \left(U_{21}^T \beta_0  \right) } \cdot U_{21}
    + \mathds{1}(Y_2 \le Y_1) \frac{\phi \left(U_{12}^T \beta_0 \right)}{\Phi \left(U_{12}^T \beta_0 \right)} \cdot U_{12}.
\end{equation*}
Note that $\psi((X_1, Y_1), (X_2, Y_2))$ is symmetric as there are no ties in $Y_i$'s ($Y$ has a continuous distribution).

We use the generalization of the classical Strong Law of Large Numbers (SLLN) for U-statistics (e.g. Theorem A in  \citep{Serfling_1980} Section 5.4). The expectation of the kernel $\psi$ is 
\begin{align*}
    \mathbb{E} [&\psi((X_1, Y_1), (X_2, Y_2))]\\
    &= 
     \mathbb{E} \left[ \mathds{1}(Y_2 > Y_2) \frac{\phi \left(U_{21}^T \beta_0  \right) }{\Phi \left(U_{21}^T \beta_0  \right) } \cdot U_{21} + \mathds{1}(Y_j \le Y_i) \frac{\phi \left(U_{12}^T \beta_0 \right)}{\Phi \left(U_{12}^T \beta_0 \right)} \cdot U_{12}  \right] \\
    &=\mathbb{E}_{X} \left[\mathbb{E}_Y \left[ \mathds{1}(Y_j > Y_i) | X \right] \frac{\phi \left(U_{21}^T \beta_0  \right) }{\Phi \left(U_{21}^T \beta_0  \right) } \cdot U_{21} +  \mathbb{E}_Y \left[ \mathds{1}(Y_j \le Y_i) | X \right] \frac{\phi \left(U_{12}^T \beta_0 \right)}{\Phi \left(U_{12}^T \beta_0 \right)} \cdot U_{12}  \right]  \\
    &= \mathbb{E}_{X} \left[ \mathbb{P} (Y_j > Y_i | X) \frac{\phi \left(U_{21}^T \beta_0  \right) }{\Phi \left(U_{21}^T \beta_0  \right) } \cdot U_{21} + \mathbb{P} (Y_j \le Y_i | X) \frac{\phi \left(U_{12}^T \beta_0 \right)}{\Phi \left(U_{12}^T \beta_0 \right)} \cdot U_{12}  \right]  \\
    &= \mathbb{E}_{X} \left[\Phi \left(U_{21}^T \beta_0  \right) \frac{\phi \left(U_{21}^T \beta_0  \right) }{\Phi \left(U_{21}^T \beta_0  \right) } \cdot U_{21} + \Phi \left(U_{12}^T \beta_0 \right) \frac{\phi \left(U_{12}^T \beta_0 \right)}{\Phi \left(U_{12}^T \beta_0 \right)} \cdot U_{12}  \right]  \\
    &= \mathbb{E}_{X} \left[\phi \left(U_{21}^T \beta_0  \right) \cdot U_{21} + \phi \left(U_{12}^T \beta_0 \right) \cdot U_{12}  \right] = \mathbb{E}_{X} \left[ \phi \left(U_{21}^T \beta_0  \right) \cdot \left[ U_{21}  + U_{12} \right] \right]  \\
    &=  \mathbb{E}_{X} \left[ \phi \left(U_{21}^T \beta_0  \right) \cdot \left[ \frac{X_j - X_i}{\sqrt{2}}  + \frac{X_i - X_j}{\sqrt{2}} \right] \right] = \mathbb{E}_X[0] = 0,
\end{align*}
where the first equalities follow from the linearity of the expectation and the tower rule of the expectation, i.e. $\mathbb{E} [Q(X, Y)] = \mathbb{E} [\mathbb{E} [Q(X, Y) | X]]$ for any function $Q$. The third equality is a classical porperty of the indicator function. The fourth equality uses the monotonicity of the function $h$. The fifth step is just a cancellation of the equal members in the fractions. Finally, the sixth equality follows from the fact that $\phi(x) = \phi(-x)$ for the probability density function $\phi$ of the standard normal distribution.

We show that the absolute value of the kernel $\psi$ has a finite expectation, since
\begin{align*}
    \mathbb{E} [&|\psi((X_1, Y_1), (X_2, Y_2))|]\\
    &\leq  \mathbb{E} \left[ \mathds{1}(Y_2 > Y_2) \frac{\phi \left(U_{21}^T \beta_0  \right) }{\Phi \left(U_{21}^T \beta_0  \right) } \cdot |U_{21}| + \mathds{1}(Y_j \le Y_i) \frac{\phi \left(U_{12}^T \beta_0 \right)}{\Phi \left(U_{12}^T \beta_0 \right)} \cdot |U_{12}|  \right] \\
    &= \mathbb{E}_{X} \left[\Phi \left(U_{21}^T \beta_0  \right) \frac{\phi \left(U_{21}^T \beta_0  \right) }{\Phi \left(U_{21}^T \beta_0  \right) } \cdot |U_{21}| + \Phi \left(U_{12}^T \beta_0 \right) \frac{\phi \left(U_{12}^T \beta_0 \right)}{\Phi \left(U_{12}^T \beta_0 \right)} \cdot |U_{12}|  \right]  \\
    &=\phi \left(U_{12}^T \beta_0 \right) \cdot \frac{2\cdot|X_i - X_j|}{\sqrt{2}} < \infty,
\end{align*}
where we used the triangle inequality for the absolute value in the first step and the other steps are similar to the calculation for the expectation of the kernel $\psi$. The last quantity is finite since the probability density function of the standard normal distribution is bounded. 

Thus, the SLLN yields
\begin{equation*}
    \nabla_{\beta} \ell_{prl}(\beta_0) \overset{a.s.}{\to} 0,
\end{equation*}
which completes the proof of the Lemma.
\end{proof}

\begin{lemma}
\label{lemma:local_max_prl}
Let $a>0$. For all $\beta \in \mathbb{R}^m$, such that $\norm{\beta - \beta_0}_2 = a$, and $n > m$ we have
\begin{equation*}
    \mathbb{P}(\ell_{prl}(\beta) < \ell_{prl}(\beta_0) ) \to 1.
\end{equation*}
\end{lemma}

\begin{proof}
The taylor expansion of $\ell_{prl}(\beta)$ around $\beta_0$ gives
\begin{equation*}
    \ell_{prl}(\beta) - \ell_{prl}(\beta_0) = (\beta - \beta_0)^T \nabla_{\beta} \ell_{prl}(\beta_0) + \frac{1}{2}  (\beta - \beta_0)^T \nabla_{\beta}^2 \ell_{prl}(\beta^*)  (\beta - \beta_0),
\end{equation*}
where $\norm{\beta^* - \beta_0}_2 \leq \norm{\beta - \beta_0}_2 = a$, i.e., $\beta^*$ is in the closed ball around $\beta_0$ with radius $a$. Clearly, $\nabla_{\beta}^2 \ell_{prl}(\beta^*)$ is continuous with respect to $\beta^*$. Moreover, from Proposition \ref{prop:prl_concave} we know that the maximum eigenvalue of $\nabla_{\beta}^2 \ell_{prl}$ is negative. Thus, there exist a maximum eigenvalue $\lambda_{max}<0$ of $\nabla_{\beta}^2 \ell_{prl}$ within the closed ball with center $\beta_0$ and radius $a$.  

Therefore, we have
\begin{equation*}
    \frac{1}{2}  (\beta - \beta_0)^T \nabla_{\beta}^2 \ell_{prl}(\beta^*)  (\beta - \beta_0) \leq \frac{\lambda_{max}}{2} \norm{\beta - \beta_0}_2^2.
\end{equation*}
Using the previous Lemma \ref{lemma:grad_prl_lik_zero_as} the remaining term $(\beta - \beta_0)^T \nabla_{\beta} \ell_{prl}(\beta_0)$ asymptotically vanishes in probability and thus the claim follows.
\end{proof}

In the following we prove Theorem \ref{thm:asm_prl}., i.e. $\hat{\beta}_{prl} - \beta_0 = o_P(1)$. 

\begin{proof}
From Lemma \ref{lemma:local_max_prl} we know that for any fixed $a > 0$, there exists a maximum of $\ell_{prl}(\beta)$ within the closed ball around $\beta_0$ with radius $a$ with probability tends to 1. However, $\hat{\beta}_{prl}$ maximizes $\ell_{prl}(\beta)$, and thus
\begin{equation*}
    \mathbb{P} \left(\norm{\hat{\beta}_{prl} - \beta_0}_2 < a \right) \to 1 \text{ as } n \to \infty,
\end{equation*}
which completes the proof.
\end{proof}

\section{PROOF OF THEOREM \ref{thm:causal_order_consistency}}
\label{app:proof_of_consistency}

To prove the Theorem we employ the following Lemma.

\begin{lemma}
\label{lemma:sink_consistency}
Under the assumption ($\mathbf{A}$) we have 
\begin{itemize}
    \item If node $k$ is not a sink node, then
    \begin{equation*}
    t_k > \xi + o_P(1).
    \end{equation*}
    \item If node $k$ is a sink node, then
    \begin{equation*}
    t_k \to 0 \text{ as } n \to \infty.
\end{equation*}
\end{itemize}
\end{lemma}

\begin{proof}
First, assume that $k$ is not a sink node. Denoting
$\hat{\varepsilon} := \hat{h}(X^{(k)}) - \left(X^{(-k)}\right)^T \hat{\beta}^{(k)}$, assumption ($\mathbf{A}$) gives that $HSIC(\mathbb{P}^{X^{(-k)}, \hat{\varepsilon}}) > \xi$. On the other hand, Theorem 3 in \cite{HSIC_2005} gives that for all $\delta > 0$, with probability at least $1 - \delta$ we have 
\begin{equation*}
    |HSIC(\{ X_j^{(-k)}, \hat{\varepsilon}_j^{(k)} \}_{j=1}^n) - HSIC(\mathbb{P}^{X^{(-k)}, \hat{\varepsilon}})| \leq \sqrt{\frac{log 6/\delta}{\alpha^2 n }} + \frac{C}{n},
\end{equation*}
where $\alpha$ and $C$ are constants. Thus, 
\begin{align*}
    t_k - \xi & = HSIC(\{ X_j^{(-k)}, \hat{\varepsilon}_j^{(k)} \}_{j=1}^n) - \xi \\
    &> HSIC(\{ X_j^{(-k)}, \hat{\varepsilon}_j^{(k)} \}_{j=1}^n) - HSIC(\mathbb{P}^{X^{(-k)}, \hat{\varepsilon}}) = o_P(1).
\end{align*}
Second, assume that $k$ is a sink node. Then $\{X_j^{(k)}, X_j^{(-k)} \}_{j=1}^n$ is an i.i.d. sample from the PNL regression model \eqref{model:pnl_regression}. Using the (point-wise) consistency of the estimators, we have 
\begin{equation*}
    \hat{h}_k(X_1^{(k)}) - h(X_1^{(k)}) = o_P(1) \text{ and } \left(X_1^{(-k)}\right)^T \hat{\beta}^{(k)} - \left(X_1^{(-k)}\right)^T \beta^{(k)} = o_P(1).
\end{equation*}
Thus, we obtain
\begin{equation}
\label{eqn:noise_consistency}
    \hat{\varepsilon}^{(k)}_1 - \varepsilon^{(k)}_1 = \hat{h}_k(X_1^{(k)}) -  h(X_1^{(k)}) - \left(\left(X_1^{(-k)}\right)^T \hat{\beta}^{(k)} - \left(X^{(-k)}\right)^T \beta^{(k)} \right) = o_P(1) - o_P(1) = o_P(1).
\end{equation}
From the definition of $t_k$ and the decomposition of the HSIC, we have 
\begin{align*}
    t_k := HSIC(\{ X_j^{(-k)}, \hat{\varepsilon}_j^{(k)} \}_{j=1}^n) = \frac{1}{n^2} \sum_{i,j=1}^n K_{ij} \hat{L}_{ij} + \frac{1}{n^4} \sum_{i,j,q,r=1}^n K_{ij} \hat{L}_{ij} - \frac{2}{n^3} \sum_{i,j,q=1}^n K_{ij} \hat{L}_{ij} = \hat{Q}_1 + \hat{Q}_2 - \hat{Q}_3,
\end{align*}
where $$K_{ij} := \exp\left(-\norm{X_i^{(-k)} - X_j^{(-k)}}_2^2 \right), \quad \hat{L}_{ij} := \exp(-(\hat{\varepsilon}_i^{(k)} - \hat{\varepsilon}_j^{(k)})^2),$$ and 
$\hat{Q}_1 := \frac{1}{n^2} \sum_{i,j=1}^n K_{ij} \hat{L}_{ij}, \quad 
\hat{Q}_2 :=  \frac{1}{n^4} \sum_{i,j,q,r=1}^n K_{ij} \hat{L}_{ij}, ,\quad \hat{Q}_3 := \frac{2}{n^3} \sum_{i,j,q=1}^n K_{ij} \hat{L}_{ij}$.

Similar to the proof of Theorem 2 in \cite{hsic_proof_part} we will show that $t_k - HSIC(\mathbb{P}^{X^{(-k)}, \varepsilon^k}) = o_P(1)$. From Lemma 1 in \cite{HSIC_2005} we have 
\begin{equation*}
    HSIC(\mathbb{P}^{X^{(-k)}, \varepsilon^k}) = Q_1 + Q_2 - Q_3,
\end{equation*}
where  $L_{ij} := \exp(-(\varepsilon_i^{(k)} - \varepsilon_j^{(k)})^2), \quad 
Q_1 := \mathbb{E}_{X_1^{(-k)}, \varepsilon_1^{(k)}, X_2^{(-k)}, \varepsilon_2^{(k)}}[K_{12}L_{12}], \quad Q_2 := \mathbb{E}_{X_1^{(-k)}, X_2^{(-k)}}[K_{12}]\mathbb{E}_{\varepsilon_1^{(k)}, \varepsilon_2^{(k)}}[L_{12}]$ and $Q_3 :=  2 \mathbb{E}_{X_1^{(-k)}, \varepsilon_1^{(k)}}[\mathbb{E}_{X_1^{(-k)}}[K_{12}] \mathbb{E}_{\varepsilon_1^{(k)}}[L_{12}]]$.

We show that $\hat{Q}_1 - Q_1 = o_P(1)$. We have
\begin{align}
    \hat{Q}_1 - Q_1 &= \frac{1}{n^2} \sum_{i,j=1}^n K_{ij} \hat{L}_{ij} - \mathbb{E}[K_{12}L_{12}] = \frac{1}{n^2} \sum_{i\neq j}^n K_{ij} \hat{L}_{ij} - \mathbb{E}[K_{12}L_{12}] + \frac{1}{n}\nonumber \\
    & = \frac{1}{n(n-1)} \sum_{i \neq j}^n K_{ij} \hat{L}_{ij} - \mathbb{E}[K_{12}L_{12}] + \frac{1}{n} - \frac{1}{n^2(n-1)}\sum_{i \neq j}^n K_{ij} \hat{L}_{ij} \nonumber \\
    \label{eqn:q1_diff}
    &= \frac{1}{n(n-1)} \sum_{i \neq j}^n K_{ij} \hat{L}_{ij} -  \mathbb{E}[K_{12}\hat{L}_{12}] + \mathbb{E}[K_{12}\hat{L}_{12}] - \mathbb{E}[K_{12}L_{12}] + \frac{1}{n} - \frac{1}{n^2(n-1)}\sum_{i \neq j}^n K_{ij} \hat{L}_{ij},
\end{align}
where we used $K_{jj} = \hat{L}_{jj} = 1$ in the second equality.
For $\delta > 0$, the Markov inequality gives
\begin{align}
    \mathbb{P} & \left(\left| \frac{1}{n(n-1)} \sum_{i \neq j}^n K_{ij} \hat{L}_{ij} - \mathbb{E}[K_{12}\hat{L}_{12}] \right| \geq \delta \right) \leq \frac{\mathbb{E} \left[\left(  \frac{1}{n(n-1)} \sum_{i \neq j}^nK_{ij} \hat{L}_{ij} - \mathbb{E}[K_{12}\hat{L}_{12}] \right)^2 \right]}{\delta^2} \nonumber \\
    & = \frac{1}{n(n-1) \delta^2}  Var(K_{12} \hat{L}_{12}) + \frac{1}{n^2(n-1)^2 \delta^2} \sum_{i \neq j} \sum_{p \neq q} Cov(K_{ij}\hat{L}_{ij}, K_{pq}\hat{L}_{pq}) \nonumber \\
    \label{eqn:q1_diff_first_term}
    & = O(1)\frac{1}{n(n-1)} + O(1) \frac{1}{n} Cov(K_{12}\hat{L}_{12}, K_{13}\hat{L}_{13}) + O(1) Cov(K_{12}\hat{L}_{12}, K_{34}\hat{L}_{34}).
\end{align}
Here, we used that $K_{ij}$ and $\hat{L}_{ij}$ are bounded by 1, thus, their variances are bounded. The number of terms in the quadruple sum that have exactly three different indices are of order $n^3$, but the denominator is of order $n^4$, which leads to the second $O(1)$ term. The last $O(1)$ comes from the fact that the number of terms in the quadruple sum that have exactly four different indices are of order $n^4$.  

Using \eqref{eqn:noise_consistency} and the continuous mapping theorem we obtain 
\begin{equation*}
    K_{12}\hat{L}_{12} - K_{12}L_{12} = o_P(1),
\end{equation*}
and since $\hat{L}_{12}$ is bounded, it is uniformly integrable and we obtain
\begin{equation*}
    \mathbb{E}[K_{12}\hat{L}_{12}] - \mathbb{E}[K_{12}L_{12}] = o(1).
\end{equation*}
In a similar way, we obtain
\begin{equation*}
    \mathbb{E}[K_{12}\hat{L}_{12}K_{34}\hat{L}_{34}] - \mathbb{E}[K_{12}L_{12} K_{34}L_{34}] = o(1).
\end{equation*}
Using the fact that $K_{12}L_{12}$ is independent from $K_{34}L_{34}$ and employing the two equalities above , we have 
\begin{align*}
    Cov(K_{12}\hat{L}_{12}, K_{34}\hat{L}_{34}) &= \mathbb{E}[K_{12}\hat{L}_{12} K_{34}\hat{L}_{34}] - \mathbb{E}[K_{12}\hat{L}_{12}] \mathbb{E}[K_{34}\hat{L}_{34}] \\
    & = \mathbb{E}[K_{12}L_{12} K_{34}L_{34}] - \mathbb{E}[K_{12}L_{12}] \mathbb{E}[K_{34}L_{34}]
    + o_P(1) = o_P(1).
\end{align*}
Thus, all terms in \eqref{eqn:q1_diff_first_term} are $o_P(1)$. Moreover, from the above arguments we have $\mathbb{E}[K_{12}\hat{L}_{12}] - \mathbb{E}[K_{12}L_{12}] = o(1)$ and $\frac{1}{n^2(n-1)}\sum_{i \neq j}^n K_{ij} \hat{L}_{ij} = o(1)$ as $K_{ij} \hat{L}_{ij}$ is bounded and the denominator is of order $n^3$. So, in \eqref{eqn:q1_diff} all the terms are $o_P(1)$, which gives
\begin{equation*}
    \hat{Q}_1 - Q_1 = o_P(1).
\end{equation*}
In the same fashion, one can show that $\hat{Q}_2 - Q_2 = o_P(1)$ and $\hat{Q}_3 - Q_3 = o_P(1)$, which put together proves that $t_k - HSIC(\mathbb{P}^{X^{(-k)}, \varepsilon^k}) = o_P(1)$. Moreover, $HSIC(\mathbb{P}^{X^{(-k)}, \varepsilon^k}) = 0$, since $X^{(k)}$ is a sink node and thus the noise $\varepsilon^k$ is independent from the remaining nodes $X^{(-k)}$. Thus, 
\begin{equation*}
    t_k \to 0 \text{ as } n \to \infty \text{ if } k \text{ is a sink node},
\end{equation*}
which completes the proof of the Lemma.
\end{proof}
Now we can prove Theorem \ref{thm:causal_order_consistency}.

For any set $A \subseteq \{ 1, 2, \dots, m \}$ we denote the sub-graph of $\mathcal{G}^0$ over the nodes $A$ as $\mathcal{G}^0_A$.
Then
\begin{align*}
    \mathbb{P}(\hat{\pi} \in \Pi^0) &= \mathbb{P}(\forall k \in \{1, 2, \dots m \}: \hat{\pi}(k) \text{ is a sink node in } \mathcal{G}^0_{ \{\hat{\pi}(1), \dots, \hat{\pi}(k) \} })\\
    &=
    1 - \mathbb{P}(\exists k \in \{1, 2, \dots m \} : \hat{\pi}(k) \text{ is not a sink node in } \mathcal{G}^0_{\{\hat{\pi}(1), \dots, \hat{\pi}(k)\}}) \\
    & \geq 1 - \sum_{k=1}^m \mathbb{P}(\hat{\pi}(k) \text{ is not a sink node in } \mathcal{G}^0_{\{\hat{\pi}(1), \dots, \hat{\pi}(k)\}}).
\end{align*}
The proof of Theorem \ref{thm:causal_order_consistency} is completed if we show that 
\begin{equation}
\label{eqn:sink_consistency_all}
    \mathbb{P}(\hat{\pi}(k) \text{ is a sink node in } \mathcal{G}^0_{\{\hat{\pi}(1), \dots, \hat{\pi}(k)\}}) \to 1 \text{ as } n \to \infty \quad \forall k \in \{1, 2, \dots, m \},
\end{equation}
since this implies
\begin{equation*}
    \mathbb{P}(\hat{\pi}(k) \text{ is not a sink node in } \mathcal{G}^0_{\{\hat{\pi}(1), \dots, \hat{\pi}(k)\}}) \to 0 \text{ as } n \to \infty  \quad \forall k \in \{1, 2, \dots, m \}.
\end{equation*}
Using assumption  ($\mathbf{A}$) and the recursive construction of $\hat{\pi}$ it is enough to prove \eqref{eqn:sink_consistency_all} for $k = m$, that is 
\begin{equation*}
\label{eqn:sink_consistency}
    \mathbb{P}(\hat{\pi}(m) \text{ is a sink node in } \mathcal{G}^0) \to 1 \text{ as } n \to \infty.
\end{equation*}
By Lemma \ref{lemma:sink_consistency} we know that $t_k$ goes to zero in probability for sink nodes and is a least $\xi$ for other nodes. Thus, 
\begin{equation*}
    \mathbb{P}(\hat{\pi} (m) \text{ is a sink node  in } \mathcal{G}^0) = \mathbb{P}(\arg \min \; \{t_k\} \text{ is a sink node  in } \mathcal{G}^0) \to 1 \text{ as } n \to \infty.
\end{equation*}
This completes the proof of the Theorem.

\section{NUMERICAL RESULTS}
\label{app:numerical_results_tables}

In this Section we provide the concrete simulation results in tabular form and additionally provide the empirical standard deviations if the results. 

Tables \ref{tab:pnl_rankg_4_gaussian}, \ref{tab:pnl_rankg_4_evd}, \ref{tab:pnl_rankg_4_logis} and \ref{tab:pnl_rankg_4_beta_10_gaussian}, \ref{tab:pnl_rankg_4_beta_10_evd}, \ref{tab:pnl_rankg_4_beta_10_logis} show the results of \texttt{RankG}, \texttt{AbPNL} and \texttt{RESIT} methods for 4-dimensional graphs with strong and weak signal settings, for Gaussian, Gumbel and Logistic noises, respectively. Tables \ref{tab:pnl_rankg_7_gaussian}, \ref{tab:pnl_rankg_7_evd}, \ref{tab:pnl_rankg_7_logis} and \ref{tab:pnl_rankg_7_beta_10_gaussian}, \ref{tab:pnl_rankg_7_beta_10_evd}, \ref{tab:pnl_rankg_7_beta_10_logis} show the results for 7-dimensional graphs. In both cases the sample sizes are 100, 500, 1000, 1500, and 2000, respectively.

Moreover, Tables \ref{tab:pnl_ranks_4_gaussian}, \ref{tab:pnl_ranks_4_evd}, \ref{tab:pnl_ranks_4_logis} and \ref{tab:pnl_ranks_4_beta_10_gaussian}, \ref{tab:pnl_ranks_4_beta_10_evd}, \ref{tab:pnl_ranks_4_beta_10_logis} show the results of \texttt{RankS}, \texttt{RankG}, \texttt{AbPNL} and \texttt{RESIT} methods for 4-dimensional graphs with strong and weak signal settings, respectively.    Tables \ref{tab:pnl_ranks_4_g4_gaussian}, \ref{tab:pnl_ranks_4_g4_evd}, \ref{tab:pnl_ranks_4_g4_logis} and \ref{tab:pnl_ranks_4_beta_10_g4_gaussian}, \ref{tab:pnl_ranks_4_beta_10_g4_evd}, \ref{tab:pnl_ranks_4_beta_10_g4_logis} show the results of \texttt{RankS}, \texttt{RankG}, \texttt{AbPNL} and \texttt{RESIT} methods for 4-dimensional graphs where the function $g$ is polynomial of degree 4 with strong and weak signal settings, respectively. In both cases the sample sizes are 100, 150, 200, 250, and 300.

\begin{table*}[!ht]
    \centering
    \begin{tabular}{||c||c|c|c||}
    \hline
    -  &  \multicolumn{3}{c||}{Gaussian noise} \\
    \hline
    \hline
      -   &  RankG & AbPNL & RESIT \\
        \hline
100 & 2.51 $\pm$ 1.26 & 3.94 $\pm$ 1.09 & \textbf{2.4 $\pm$ 1.16}\\
500 & \textbf{1.86 $\pm$ 1.31} & 3.46 $\pm$ 1.08 & 2.41 $\pm$ 1.23\\
1000 & \textbf{1.61 $\pm$ 1.15} & 3.14 $\pm$ 1.14 & 2.43 $\pm$ 1.16\\
1500 & \textbf{1.66 $\pm$ 1.35} & 3.28 $\pm$ 1.18 & 2.62 $\pm$ 1.36\\
2000 & \textbf{1.63 $\pm$ 1.17} & 3.23 $\pm$ 1.12 & 3.15 $\pm$ 1.51\\
        \hline
    \end{tabular}
    \caption{Results of RankG, RESIT and AbPNL methods on 4 nodes with $\beta \sim U(-100, 100)$ for Gaussian noise ($100$ repetitions).}
    \label{tab:pnl_rankg_4_gaussian}
\end{table*}

\begin{table*}[!ht]
    \centering
    \begin{tabular}{||c||c|c|c||}
    \hline
    -  &  \multicolumn{3}{c||}{Gumbel noise} \\
    \hline
    \hline
      -   &  RankG & AbPNL & RESIT \\
        \hline
100 & 2.86 $\pm$ 1.36 & 3.98 $\pm$ 1.36 & \textbf{2.3 $\pm$ 1.18}\\
500 & \textbf{2.34 $\pm$ 1.45} & 3.07 $\pm$ 1.24 & 2.49 $\pm$ 1.2\\
1000 & \textbf{1.8 $\pm$ 1.28} & 3.31 $\pm$ 1.11 & 2.43 $\pm$ 1.28\\
1500 & \textbf{1.68 $\pm$ 1.1} & 3.06 $\pm$ 1.03 & 3.04 $\pm$ 1.31\\
2000 & \textbf{1.37 $\pm$ 1.14} & 3.17 $\pm$ 1.18 & 3.02 $\pm$ 1.41\\
        \hline
    \end{tabular}
    \caption{Results of RankG, RESIT and AbPNL methods on 4 nodes with $\beta \sim U(-100, 100)$ for Gumbel noise ($100$ repetitions).}
    \label{tab:pnl_rankg_4_evd}
\end{table*}

\begin{table*}[!ht]
    \centering
    \begin{tabular}{||c||c|c|c||}
    \hline
    -  &  \multicolumn{3}{c||}{Logistic noise} \\
    \hline
    \hline
      -   &  RankG & AbPNL & RESIT \\
        \hline
100 & 2.61 $\pm$ 1.35 & 3.89 $\pm$ 1.11 & \textbf{2.31 $\pm$ 1.1}\\
500 & \textbf{2.06 $\pm$ 1.34} & 3.32 $\pm$ 1.24 & 2.36 $\pm$ 1.13\\
1000 & \textbf{1.85 $\pm$ 1.37} & 3.41 $\pm$ 1.1 & 2.66 $\pm$ 1.24\\
1500 & \textbf{1.68 $\pm$ 1.38} & 3.27 $\pm$ 1.2 & 2.92 $\pm$ 1.32\\
2000 & \textbf{1.79 $\pm$ 1.18} & 3.29 $\pm$ 1.16 & 3.44 $\pm$ 1.38\\
        \hline
    \end{tabular}
    \caption{Results of RankG, RESIT and AbPNL methods on 4 nodes with $\beta \sim U(-100, 100)$ for Logistic noise ($100$ repetitions).}
    \label{tab:pnl_rankg_4_logis}
\end{table*}

\begin{table*}[!ht]
    \centering
    \begin{tabular}{||c||c|c|c||}
    \hline
    -  &  \multicolumn{3}{c||}{Gaussian noise} \\
    \hline
    \hline
      -   &  RankG & AbPNL & RESIT \\
        \hline
100 & \textbf{1.21 $\pm$ 1.23} & 3.7 $\pm$ 1.28 & 2.72 $\pm$ 1.2\\
500 & \textbf{0.21 $\pm$ 0.56} & 2.8 $\pm$ 1.01 & 2.69 $\pm$ 1.12\\
1000 & \textbf{0.15 $\pm$ 0.59} & 3.18 $\pm$ 1.08 & 2.73 $\pm$ 1.08\\
1500 & \textbf{0.04 $\pm$ 0.24} & 3.06 $\pm$ 1.08 & 2.9 $\pm$ 1.06\\
2000 & \textbf{0.05 $\pm$ 0.3} & 3.26 $\pm$ 1.17 & 3.1 $\pm$ 1.29\\
        \hline
    \end{tabular}
    \caption{Results of RankG, RESIT and AbPNL methods on 4 nodes with $\beta \sim U(-10, 10)$ for Gaussian noise ($100$ repetitions).}
    \label{tab:pnl_rankg_4_beta_10_gaussian}
\end{table*}

\begin{table*}[!ht]
    \centering
    \begin{tabular}{||c||c|c|c||}
    \hline
    -  &  \multicolumn{3}{c||}{Gumbel noise} \\
    \hline
    \hline
      -   &  RankG & AbPNL & RESIT \\
        \hline
100 & \textbf{1.44 $\pm$ 1.26} & 3.47 $\pm$ 1.13 & 2.86 $\pm$ 1.26\\
500 & \textbf{0.42 $\pm$ 0.91} & 2.61 $\pm$ 1.1 & 2.62 $\pm$ 1.19\\
1000 & \textbf{0.13 $\pm$ 0.39} & 2.88 $\pm$ 1.15 & 2.75 $\pm$ 1.15\\
1500 & \textbf{0.15 $\pm$ 0.54} & 2.76 $\pm$ 1.14 & 2.98 $\pm$ 1.2\\
2000 & \textbf{0.08 $\pm$ 0.37} & 2.83 $\pm$ 1.21 & 3.29 $\pm$ 1.26\\
        \hline
    \end{tabular}
    \caption{Results of RankG, RESIT and AbPNL methods on 4 nodes with $\beta \sim U(-10, 10)$ for Gumbel noise ($100$ repetitions).}
    \label{tab:pnl_rankg_4_beta_10_evd}
\end{table*}

\begin{table*}[!ht]
    \centering
    \begin{tabular}{||c||c|c|c||}
    \hline
    -  &  \multicolumn{3}{c||}{Logistic noise} \\
    \hline
    \hline
      -   &  RankG & AbPNL & RESIT \\
        \hline
100 & \textbf{1.05 $\pm$ 1.3} & 3.82 $\pm$ 1.15 & 2.5 $\pm$ 1.24\\
500 & \textbf{0.39 $\pm$ 0.79} & 2.84 $\pm$ 1.02 & 2.46 $\pm$ 1\\
1000 & \textbf{0.14 $\pm$ 0.47} & 3.01 $\pm$ 1.14 & 2.53 $\pm$ 1.11\\
1500 & \textbf{0 $\pm$ 0} & 2.95 $\pm$ 1.18 & 2.98 $\pm$ 1.05\\
2000 & \textbf{0 $\pm$ 0} & 2.69 $\pm$ 1.17 & 3.28 $\pm$ 1.21\\
        \hline
    \end{tabular}
    \caption{Results of RankG, RESIT and AbPNL methods on 4 nodes with $\beta \sim U(-10, 10)$ for Logistic noise ($100$ repetitions).}
    \label{tab:pnl_rankg_4_beta_10_logis}
\end{table*}


\begin{table*}[!ht]
    \centering
    \begin{tabular}{||c||c|c|c||}
    \hline
    -  &  \multicolumn{3}{c||}{Gaussian noise} \\
    \hline
    \hline
      -   &  RankG & AbPNL & RESIT \\
        \hline
100 & 4.65 $\pm$ 2.21 & 6.72 $\pm$ 2.13 & \textbf{3.89 $\pm$ 1.85}\\
500 & \textbf{3.75 $\pm$ 2.11} & 5.4 $\pm$ 1.88 & 4.1 $\pm$ 1.76\\
1000 & \textbf{3.29 $\pm$ 2.03} & 5.46 $\pm$ 1.96 & 4.27 $\pm$ 1.95\\
1500 & \textbf{2.94 $\pm$ 2.13} & 5.4 $\pm$ 1.87 & 4.34 $\pm$ 1.88\\
2000 & \textbf{3.35 $\pm$ 2.09} & 5.49 $\pm$ 1.84 & 5.02 $\pm$ 2.09\\
        \hline
    \end{tabular}
    \caption{Results of RankG, RESIT and AbPNL methods on 7 nodes with $\beta \sim U(-100, 100)$ for Gaussian noise ($100$ repetitions).}
    \label{tab:pnl_rankg_7_gaussian}
\end{table*}

\begin{table*}[!ht]
    \centering
    \begin{tabular}{||c||c|c|c||}
    \hline
    -  &  \multicolumn{3}{c||}{Gumbel noise} \\
    \hline
    \hline
      -   &  RankG & AbPNL & RESIT \\
        \hline
100 & 4.74 $\pm$ 2.12 & 6.33 $\pm$ 2.26 & \textbf{3.95 $\pm$ 2.09}\\
500 & \textbf{3.68 $\pm$ 1.97} & 5.41 $\pm$ 1.83 & 3.95 $\pm$ 1.75\\
1000 & \textbf{3.4 $\pm$ 1.98} & 5.39 $\pm$ 1.92 & 3.78 $\pm$ 1.92\\
1500 & \textbf{3.33 $\pm$ 2.05} & 5.96 $\pm$ 2.12 & 4.66 $\pm$ 1.99\\
2000 & \textbf{3.47 $\pm$ 2.61} & 5.26 $\pm$ 1.7 & 5.07 $\pm$ 2.41\\
        \hline
    \end{tabular}
    \caption{Results of RankG, RESIT and AbPNL methods on 7 nodes with $\beta \sim U(-100, 100)$ for Gumbel noise ($100$ repetitions).}
    \label{tab:pnl_rankg_7_evd}
\end{table*}

\begin{table*}[!ht]
    \centering
    \begin{tabular}{||c||c|c|c||}
    \hline
    -  &  \multicolumn{3}{c||}{Logistic noise} \\
    \hline
    \hline
      -   &  RankG & AbPNL & RESIT \\
        \hline
100 & 4.56 $\pm$ 2.1 & 6.45 $\pm$ 2.08 & \textbf{4.01 $\pm$ 1.96}\\
500 & \textbf{3.65 $\pm$ 2.13} & 5.52 $\pm$ 1.49 & 3.86 $\pm$ 1.85\\
1000 & \textbf{3.88 $\pm$ 2.42} & 5.35 $\pm$ 1.96 & 4.18 $\pm$ 1.88\\
1500 & \textbf{2.93 $\pm$ 1.98} & 5.37 $\pm$ 1.8 & 4.42 $\pm$ 2.01\\
2000 & \textbf{3.39 $\pm$ 1.97} & 5.1 $\pm$ 1.64 & 4.96 $\pm$ 2.17\\
        \hline
    \end{tabular}
    \caption{Results of RankG, RESIT and AbPNL methods on 7 nodes with $\beta \sim U(-100, 100)$ for Logistic noise($100$ repetitions).}
    \label{tab:pnl_rankg_7_logis}
\end{table*}

\begin{table*}[!ht]
    \centering
    \begin{tabular}{||c||c|c|c||}
    \hline
    -  &  \multicolumn{3}{c||}{Gaussian noise} \\
    \hline
    \hline
      -   &  RankG & AbPNL & RESIT \\
        \hline
100 & \textbf{2.44 $\pm$ 1.69} & 6.71 $\pm$ 2.47 & 4.15 $\pm$ 1.75\\
500 & \textbf{1.11 $\pm$ 1.46} & 5.06 $\pm$ 1.75 & 4.4 $\pm$ 1.84\\
1000 & \textbf{0.28 $\pm$ 0.75} & 5.08 $\pm$ 1.87 & 4.47 $\pm$ 1.88\\
1500 & \textbf{0.19 $\pm$ 0.6} & 5.18 $\pm$ 2 & 4.35 $\pm$ 1.79\\
2000 & \textbf{0.19 $\pm$ 0.69} & 5.26 $\pm$ 1.55 & 4.69 $\pm$ 2\\
        \hline
    \end{tabular}
    \caption{Results of RankG, RESIT and AbPNL methods on 7 nodes with $\beta \sim U(-10, 10)$ for Gaussian noise ($100$ repetitions).}
    \label{tab:pnl_rankg_7_beta_10_gaussian}
\end{table*}

\begin{table*}[!ht]
    \centering
    \begin{tabular}{||c||c|c|c||}
    \hline
    -  &  \multicolumn{3}{c||}{Gumbel noise} \\
    \hline
    \hline
      -   &  RankG & AbPNL & RESIT \\
        \hline
100 & \textbf{3.13 $\pm$ 2.02} & 6.11 $\pm$ 2.22 & 4.17 $\pm$ 1.76\\
500 & \textbf{0.88 $\pm$ 1.4} & 4.32 $\pm$ 1.65 & 4.41 $\pm$ 1.74\\
1000 & \textbf{0.38 $\pm$ 0.93} & 4.7 $\pm$ 1.87 & 4.22 $\pm$ 1.97\\
1500 & \textbf{0.37 $\pm$ 0.97} & 4.88 $\pm$ 1.78 & 4.88 $\pm$ 1.87\\
2000 & \textbf{0.14 $\pm$ 0.55} & 4.81 $\pm$ 1.94 & 5.02 $\pm$ 2.06\\
        \hline
    \end{tabular}
    \caption{Results of RankG, RESIT and AbPNL methods on 7 nodes with $\beta \sim U(-10, 10)$ for Gumbel noise  ($100$ repetitions).}
    \label{tab:pnl_rankg_7_beta_10_evd}
\end{table*}

\begin{table*}[!ht]
    \centering
    \begin{tabular}{||c||c|c|c||}
    \hline
    -  &  \multicolumn{3}{c||}{Logistic noise} \\
    \hline
    \hline
      -   &  RankG & AbPNL & RESIT \\
        \hline
100 & \textbf{2.65 $\pm$ 2} & 6.83 $\pm$ 2.2 & 4.03 $\pm$ 1.54\\
500 & \textbf{0.68 $\pm$ 1.09} & 4.87 $\pm$ 1.82 & 4.12 $\pm$ 1.77\\
1000 & \textbf{0.45 $\pm$ 1.01} & 5.08 $\pm$ 1.79 & 4.27 $\pm$ 1.8\\
1500 & \textbf{0.34 $\pm$ 0.98} & 4.89 $\pm$ 1.87 & 4.55 $\pm$ 1.76\\
2000 & \textbf{0.2 $\pm$ 0.68} & 4.43 $\pm$ 1.69 & 4.56 $\pm$ 1.81\\
        \hline
    \end{tabular}
    \caption{Results of RankG, RESIT and AbPNL methods on 7 nodes with $\beta \sim U(-10, 10)$ for Logistic noise ($100$ repetitions).}
    \label{tab:pnl_rankg_7_beta_10_logis}
\end{table*}






\begin{table*}[!ht]
    \centering
    \begin{tabular}{||c||c|c|c|c||}
    \hline
    -  &  \multicolumn{4}{c||}{Gaussian noise} \\
    \hline
    \hline
      -   &  RankS & RankG & AbPNL & RESIT \\
        \hline
100 & \textbf{2.23 $\pm$ 1.42} & 2.51 $\pm$ 1.26 & 3.94 $\pm$ 1.09 & 2.4 $\pm$ 1.16\\
150 & \textbf{1.73 $\pm$ 1.49} & 2.75 $\pm$ 1.25 & 2.6 $\pm$ 1.2 & 2.47 $\pm$ 1.13\\
200 & \textbf{1.63 $\pm$ 1.34} & 2.24 $\pm$ 1.32 & 2.87 $\pm$ 1.19 & 2.3 $\pm$ 1.02\\
250 & \textbf{1.56 $\pm$ 1.27} & 2.3 $\pm$ 1.16 & 3.07 $\pm$ 1.15 & 2.4 $\pm$ 1.22\\
300 & \textbf{1.49 $\pm$ 1.29} & 2.24 $\pm$ 1.35 & 2.97 $\pm$ 1.02 & 2.54 $\pm$ 1.1\\
        \hline
    \end{tabular}
    \caption{Results of RankS, RankG, RESIT and AbPNL methods on 4 nodes with $\beta \sim U(-100, 100)$ for Gaussian noise ($100$ repetitions).}
    \label{tab:pnl_ranks_4_gaussian}
\end{table*}

\begin{table*}[!ht]
    \centering
    \begin{tabular}{||c||c|c|c|c||}
    \hline
    -  &  \multicolumn{4}{c||}{Gumbel noise} \\
    \hline
    \hline
      -   &  RankS & RankG & AbPNL & RESIT \\
        \hline
100 & \textbf{1.69 $\pm$ 1.36} & 2.86 $\pm$ 1.36 & 3.98 $\pm$ 1.36 & 2.3 $\pm$ 1.18\\
150 & \textbf{1.76 $\pm$ 1.28} & 2.56 $\pm$ 1.4 & 2.88 $\pm$ 1.34 & 2.37 $\pm$ 1.1\\
200 & \textbf{1.77 $\pm$ 1.38} & 2.78 $\pm$ 1.3 & 2.77 $\pm$ 1.17 & 2.38 $\pm$ 1.25\\
250 & \textbf{1.73 $\pm$ 1.27} & 2.44 $\pm$ 1.18 & 3.04 $\pm$ 1.13 & 2.45 $\pm$ 1.16\\
300 & \textbf{1.84 $\pm$ 1.23} & 2.4 $\pm$ 1.43 & 3.11 $\pm$ 1.12 & 2.35 $\pm$ 1.1\\
        \hline
    \end{tabular}
    \caption{Results of RankS, RankG, RESIT and AbPNL methods on 4 nodes with $\beta \sim U(-100, 100)$ for Gumbel noise ($100$ repetitions).}
    \label{tab:pnl_ranks_4_evd}
\end{table*}

\begin{table*}[!ht]
    \centering
    \begin{tabular}{||c||c|c|c|c||}
    \hline
    -  &  \multicolumn{4}{c||}{Logistic noise} \\
    \hline
    \hline
      -   &  RankS & RankG & AbPNL & RESIT \\
        \hline
100 & \textbf{2.02 $\pm$ 1.32} & 2.61 $\pm$ 1.35 & 3.89 $\pm$ 1.11 & 2.31 $\pm$ 1.1\\
150 & \textbf{1.8 $\pm$ 1.26} & 2.58 $\pm$ 1.39 & 2.84 $\pm$ 1.26 & 2.7 $\pm$ 1.23\\
200 & \textbf{1.63 $\pm$ 1.28} & 2.5 $\pm$ 1.4 & 3.17 $\pm$ 1.25 & 2.55 $\pm$ 1.25\\
250 & \textbf{1.45 $\pm$ 1.23} & 2.21 $\pm$ 1.3 & 3.32 $\pm$ 1.23 & 2.3 $\pm$ 1.25\\
300 & \textbf{1.61 $\pm$ 1.32} & 2.3 $\pm$ 1.23 & 3 $\pm$ 1.04 & 2.4 $\pm$ 1.17\\
        \hline
    \end{tabular}
    \caption{Results of RankS, RankG, RESIT and AbPNL methods on 4 nodes with $\beta \sim U(-100, 100)$ for Logistic noise ($100$ repetitions).}
    \label{tab:pnl_ranks_4_logis}
\end{table*}

\begin{table*}[!ht]
    \centering
    \begin{tabular}{||c||c|c|c|c||}
    \hline
    -  &  \multicolumn{4}{c||}{Gaussian noise} \\
    \hline
    \hline
      -   &  RankS & RankG & AbPNL & RESIT \\
        \hline
100 & 1.72 $\pm$ 1.45 & \textbf{1.21 $\pm$ 1.23} & 3.7 $\pm$ 1.28 & 2.72 $\pm$ 1.2\\
150 & 1.54 $\pm$ 1.36 & \textbf{1.11 $\pm$ 1.48} & 2.3 $\pm$ 1.41 & 2.69 $\pm$ 1.2\\
200 & 1.37 $\pm$ 1.37 & \textbf{0.71 $\pm$ 1.09} & 2.52 $\pm$ 1.29 & 2.72 $\pm$ 1.05\\
250 & 1.37 $\pm$ 1.28 & \textbf{0.77 $\pm$ 1.14} & 2.55 $\pm$ 1.07 & 2.78 $\pm$ 1.22\\
300 & 1.21 $\pm$ 1.27 & \textbf{0.48 $\pm$ 0.89} & 3.1 $\pm$ 1.28 & 2.57 $\pm$ 1.17\\
        \hline
    \end{tabular}
    \caption{Results of RankS, RankG, RESIT and AbPNL methods on 4 nodes with $\beta \sim U(-10, 10)$ for Gaussian noise ($100$ repetitions).}
    \label{tab:pnl_ranks_4_beta_10_gaussian}
\end{table*}

\begin{table*}[!ht]
    \centering
    \begin{tabular}{||c||c|c|c|c||}
    \hline
    -  &  \multicolumn{4}{c||}{Gumbel noise} \\
    \hline
    \hline
      -   &  RankS & RankG & AbPNL & RESIT \\
        \hline
100 & 1.73 $\pm$ 1.43 & \textbf{1.44 $\pm$ 1.26} & 3.47 $\pm$ 1.13 & 2.86 $\pm$ 1.26\\
150 & 1.63 $\pm$ 1.37 & \textbf{1.22 $\pm$ 1.33} & 2.04 $\pm$ 1.19 & 2.64 $\pm$ 1.28\\
200 & 1.41 $\pm$ 1.31 & \textbf{0.81 $\pm$ 1.02} & 2.5 $\pm$ 1.08 & 2.7 $\pm$ 1.32\\
250 & 1.48 $\pm$ 1.26 & \textbf{0.62 $\pm$ 1.08} & 2.38 $\pm$ 1.25 & 2.58 $\pm$ 1.24\\
300 & 1.72 $\pm$ 1.42 & \textbf{0.69 $\pm$ 1.05} & 2.44 $\pm$ 1.19 & 2.76 $\pm$ 1.12\\
        \hline
    \end{tabular}
    \caption{Results of RankS, RankG, RESIT and AbPNL methods on 4 nodes with $\beta \sim U(-10, 10)$ for Gumbel noise ($100$ repetitions).}
    \label{tab:pnl_ranks_4_beta_10_evd}
\end{table*}

\begin{table*}[!ht]
    \centering
    \begin{tabular}{||c||c|c|c|c||}
    \hline
    -  &  \multicolumn{4}{c||}{Logistic noise} \\
    \hline
    \hline
      -   &  RankS & RankG & AbPNL & RESIT \\
        \hline
100 & 1.78 $\pm$ 1.4 & \textbf{1.05 $\pm$ 1.3} & 3.82 $\pm$ 1.15 & 2.5 $\pm$ 1.24\\
150 & 1.44 $\pm$ 1.34 & \textbf{0.74 $\pm$ 1.02} & 2.86 $\pm$ 1.25 & 2.56 $\pm$ 1.16\\
200 & 1.55 $\pm$ 1.27 & \textbf{0.57 $\pm$ 1.02} & 3.18 $\pm$ 1.24 & 2.37 $\pm$ 1.17\\
250 & 1.3 $\pm$ 1.14 & \textbf{0.36 $\pm$ 0.72} & 3.14 $\pm$ 1.25 & 2.28 $\pm$ 1.09\\
300 & 1.38 $\pm$ 1.43 & \textbf{0.33 $\pm$ 0.75} & 3.23 $\pm$ 1.15 & 2.51 $\pm$ 1.11\\
        \hline
    \end{tabular}
    \caption{Results of RankS, RankG, RESIT and AbPNL methods on 4 nodes with $\beta \sim U(-10, 10)$ for Logistic noise ($100$ repetitions).}
    \label{tab:pnl_ranks_4_beta_10_logis}
\end{table*}


\begin{table*}[!ht]
    \centering
    \begin{tabular}{||c||c|c|c|c||}
    \hline
    -  &  \multicolumn{4}{c||}{Gaussian noise} \\
    \hline
    \hline
      -   &  RankS & RankG & AbPNL & RESIT \\
        \hline
100 & \textbf{1.24 $\pm$ 0.92} & 3.33 $\pm$ 1.41 & 3.7 $\pm$ 1.32 & 2.75 $\pm$ 1.31\\
150 & \textbf{1.42 $\pm$ 1.07} & 2.96 $\pm$ 1.5 & 3.1 $\pm$ 1.49 & 2.56 $\pm$ 1.21\\
200 & \textbf{1.37 $\pm$ 1.1} & 3.26 $\pm$ 1.54 & 3.37 $\pm$ 1.34 & 2.42 $\pm$ 1.32\\
250 & \textbf{1.78 $\pm$ 1.18} & 2.8 $\pm$ 1.55 & 3.35 $\pm$ 1.23 & 2.33 $\pm$ 1.12\\
300 & \textbf{1.63 $\pm$ 1} & 3.32 $\pm$ 1.59 & 3.52 $\pm$ 1.23 & 2.64 $\pm$ 1.31\\
        \hline
    \end{tabular}
    \caption{Results of RankS, RankG, RESIT and AbPNL methods on 4 nodes with $\beta \sim U(-100, 100)$ for Gaussian noise and quartic polynomial $g$ ($100$ repetitions).}
    \label{tab:pnl_ranks_4_g4_gaussian}
\end{table*}

\begin{table*}[!ht]
    \centering
    \begin{tabular}{||c||c|c|c|c||}
    \hline
    -  &  \multicolumn{4}{c||}{Gumbel noise} \\
    \hline
    \hline
      -   &  RankS & RankG & AbPNL & RESIT \\
        \hline
100 & \textbf{1.11 $\pm$ 0.92} & 3 $\pm$ 1.37 & 3.33 $\pm$ 1.36 & 2.59 $\pm$ 1.23\\
150 & \textbf{1.51 $\pm$ 1.01} & 3.14 $\pm$ 1.37 & 2.7 $\pm$ 1.27 & 2.78 $\pm$ 1.2\\
200 & \textbf{1.56 $\pm$ 0.96} & 3.29 $\pm$ 1.41 & 3.22 $\pm$ 1.22 & 2.81 $\pm$ 1.26\\
250 & \textbf{1.63 $\pm$ 0.93} & 3.52 $\pm$ 1.41 & 3.34 $\pm$ 1.08 & 2.94 $\pm$ 1.29\\
300 & \textbf{1.86 $\pm$ 0.98} & 3.41 $\pm$ 1.44 & 3.61 $\pm$ 1.15 & 2.73 $\pm$ 1.25\\
        \hline
    \end{tabular}
    \caption{Results of RankS, RankG, RESIT and AbPNL methods on 4 nodes with $\beta \sim U(-100, 100)$ for Gumbel noise and quartic polynomial $g$ ($100$ repetitions).}
    \label{tab:pnl_ranks_4_g4_evd}
\end{table*}

\begin{table*}[!ht]
    \centering
    \begin{tabular}{||c||c|c|c|c||}
    \hline
    -  &  \multicolumn{4}{c||}{Logistic noise} \\
    \hline
    \hline
      -   &  RankS & RankG & AbPNL & RESIT \\
        \hline
100 & \textbf{2.02 $\pm$ 1.32} & 2.61 $\pm$ 1.35 & 3.52 $\pm$ 1.24 & 2.31 $\pm$ 1.1\\
150 & \textbf{1.8 $\pm$ 1.26} & 2.58 $\pm$ 1.39 & 3.23 $\pm$ 1.18 & 2.7 $\pm$ 1.23\\
200 & \textbf{1.63 $\pm$ 1.28} & 2.5 $\pm$ 1.4 & 3.47 $\pm$ 1.09 & 2.55 $\pm$ 1.25\\
250 & \textbf{1.45 $\pm$ 1.23} & 2.21 $\pm$ 1.3 & 3.35 $\pm$ 1.07 & 2.3 $\pm$ 1.25\\
300 & \textbf{1.61 $\pm$ 1.32} & 2.3 $\pm$ 1.23 & 3.45 $\pm$ 1.41 & 2.4 $\pm$ 1.17\\
        \hline
    \end{tabular}
    \caption{Results of RankS, RankG, RESIT and AbPNL methods on 4 nodes with $\beta \sim U(-100, 100)$ for Logistic noise and quartic polynomial $g$ ($100$ repetitions).}
    \label{tab:pnl_ranks_4_g4_logis}
\end{table*}

\begin{table*}[!ht]
    \centering
    \begin{tabular}{||c||c|c|c|c||}
    \hline
    -  &  \multicolumn{4}{c||}{Gaussian noise} \\
    \hline
    \hline
      -   &  RankS & RankG & AbPNL & RESIT \\
        \hline
100 & \textbf{2.01 $\pm$ 1.38} & 2.92 $\pm$ 1.66 & 3.64 $\pm$ 1.28 & 2.63 $\pm$ 1.17\\
150 & \textbf{2.34 $\pm$ 1.29} & 2.75 $\pm$ 1.72 & 2.64 $\pm$ 1.43 & 2.8 $\pm$ 1.29\\
200 & \textbf{2.44 $\pm$ 1.42} & 2.7 $\pm$ 1.62 & 3.19 $\pm$ 1.35 & 2.81 $\pm$ 1.32\\
250 & \textbf{2.53 $\pm$ 1.49} & 2.42 $\pm$ 1.65 & 3.57 $\pm$ 1.27 & 2.63 $\pm$ 1.24\\
300 & \textbf{2.48 $\pm$ 1.57} & 2.79 $\pm$ 1.8 & 3.7 $\pm$ 1.2 & 2.58 $\pm$ 1.34\\
        \hline
    \end{tabular}
    \caption{Results of RankS, RankG, RESIT and AbPNL methods on 4 nodes with $\beta \sim U(-10, 10)$ for Gaussian noise and quartic polynomial $g$  ($100$ repetitions).}
    \label{tab:pnl_ranks_4_beta_10_g4_gaussian}
\end{table*}

\begin{table*}[!ht]
    \centering
    \begin{tabular}{||c||c|c|c|c||}
    \hline
    -  &  \multicolumn{4}{c||}{Gumbel noise} \\
    \hline
    \hline
      -   &  RankS & RankG & AbPNL & RESIT \\
        \hline
100 & \textbf{1.85 $\pm$ 1.17} & 2.99 $\pm$ 1.49 & 3.55 $\pm$ 1.13 & 2.83 $\pm$ 1.23\\
150 & \textbf{2.15 $\pm$ 1.21} & 2.68 $\pm$ 1.56 & 2.54 $\pm$ 1.43 & 2.84 $\pm$ 1.39\\
200 & \textbf{2.43 $\pm$ 1.37} & 2.69 $\pm$ 1.59 & 2.71 $\pm$ 1.27 & 2.87 $\pm$ 1.3\\
250 & \textbf{2.51 $\pm$ 1.33} & 2.9 $\pm$ 1.61 & 3.2 $\pm$ 1.26 & 3.05 $\pm$ 1.45\\
300 & \textbf{2.32 $\pm$ 1.35} & 3.03 $\pm$ 1.38 & 3.1 $\pm$ 1.18 & 2.94 $\pm$ 1.32\\
        \hline
    \end{tabular}
    \caption{Results of RankS, RankG, RESIT and AbPNL methods on 4 nodes with $\beta \sim U(-10, 10)$ for Gumbel noise and quartic polynomial $g$  ($100$ repetitions).}
    \label{tab:pnl_ranks_4_beta_10_g4_evd}
\end{table*}

\begin{table*}[!ht]
    \centering
    \begin{tabular}{||c||c|c|c|c||}
    \hline
    -  &  \multicolumn{4}{c||}{Logistic noise} \\
    \hline
    \hline
      -   &  RankS & RankG & AbPNL & RESIT \\
        \hline
100 & \textbf{2.12 $\pm$ 1.17} & 2.97 $\pm$ 1.45 & 3.49 $\pm$ 1.36 & 2.61 $\pm$ 1.38\\
150 & \textbf{2.29 $\pm$ 1.39} & 2.93 $\pm$ 1.42 & 3.09 $\pm$ 1.39 & 2.63 $\pm$ 1.19\\
200 & \textbf{2.4 $\pm$ 1.36} & 2.91 $\pm$ 1.54 & 3.5 $\pm$ 1.17 & 2.45 $\pm$ 1.2\\
250 & \textbf{2.34 $\pm$ 1.34} & 2.81 $\pm$ 1.5 & 3.34 $\pm$ 1.23 & 2.53 $\pm$ 1.15\\
300 & \textbf{2.48 $\pm$ 1.44} & 2.76 $\pm$ 1.72 & 3.29 $\pm$ 1.19 & 2.59 $\pm$ 1.24\\
        \hline
    \end{tabular}
    \caption{Results of RankS, RankG, RESIT and AbPNL methods on 4 nodes with $\beta \sim U(-10, 10)$ for Logistic noise and quartic polynomial $g$  ($100$ repetitions).}
    \label{tab:pnl_ranks_4_beta_10_g4_logis}
\end{table*}

\vfill

\end{document}